\documentclass{scrartcl} 

\usepackage[T1]{fontenc}
\usepackage{graphicx}
\usepackage{floatflt}
\usepackage{pgfplots}
\usepackage{bm}
\usepackage{rotating}
\usepackage{multirow}
\usepackage{booktabs}
\usepackage{xspace}
\usepackage{color}
\usepackage{braket}
\usepackage{algpseudocode}
\usepackage{tabularx}
\usepackage{pdfcomment}

\usepackage{amsmath} 
\usepackage{amssymb} 
\usepackage{amsthm}
\usepackage{bm}
\usepackage{rotating}
\usepackage{multirow}
\usepackage{booktabs}
\usepackage{placeins}
\usepgfplotslibrary{external} 
\usepackage{wrapfig}
\usepackage{graphicx}
\graphicspath{{./figures/}}
\usepackage{enumerate}
\usepackage{mathtools}
\usepackage{cleveref}
\usepackage{subcaption}
\usepackage{verbatim}

\usepackage{lipsum}
\usepackage{todonotes}
\usepackage{standalone}
\usepackage{hyperref}
\usepackage{url}
\usepackage{pgfplots}
\usepackage[section]{algorithm}
\usepackage{algpseudocode}

\newcommand{\indep}{\rotatebox[origin=c]{90}{$\models$}}

\DeclarePairedDelimiterX{\ExpArg}[1]{[}{]}{#1}
\newcommand{\Exp}{\ExpOp\ExpArg*}

\theoremstyle{plain}
\newtheorem{theorem}{Theorem}[section]
\newtheorem{lemma}[theorem]{Lemma}
\newtheorem{cor}{Corollary}

\theoremstyle{definition}
\newtheorem{definition}{Definition}[section]

\theoremstyle{remark}

\definecolor{lightgray}{gray}{0.95}

\newcommand{\Trace }[1]{\mbox{}{\bf{Tr}}\left(#1\right)}

\newcommand{\Norm}[1]{\mbox{}\left\|#1\right\|}
\newcommand{\NormS}[1]{\mbox{}\left\|#1\right\|^2}
\newcommand{\vectorize}[1]{\mathbf{vec}\left( #1 \right)}

\newcommand{\setlinespacing}[1]%
           {\setlength{\baselineskip}{#1 \defbaselineskip}}

\newcommand{\BlockDiagk}[1]{\mbox{}\left(%
\begin{array}{cc}
  \Sigma_{k} & \bf{0} \\
  \bf{0} &  \Sigma_{\rho-k}\\
\end{array}\right)}

\newcommand{\BlockDiagkk}[1]{\mbox{}\left(%
\begin{array}{cc}
  \Sigma_{k} & \bf{0} \\
  \bf{0} & \bf{0} \\
\end{array}\right)}

\newcommand{\BlockDiagkrk}[1]{\mbox{}\left(%
\begin{array}{cc}
  \bf{0} & \bf{0} \\
  \bf{0} & \Sigma_{\rho-k} \\
\end{array}\right)}

\newcommand{\BlockDiagkkh}[1]{\mbox{}\left(%
\begin{array}{c}
  \Sigma_{k} \\
  \bf{0} \\
\end{array}\right)}

\newcommand{\BlockDiagkrkh}[1]{\mbox{}\left(%
\begin{array}{c}
  \bf{0} \\
  \Sigma_{\rho-k} \\
\end{array}\right)}

\long\def\killtext#1{}

\def\Exp{\hbox{\bf{E}}}

\def\diag{\hbox{\bf{diag}}}

\RequirePackage{natbib}
\title{Supervised Dimensionality Reduction via Distance Correlation Maximization}

\author{Praneeth Vepakomma\thanks{Authors contributed equally.}\\
Department of Statistics,\\
Rutgers University,\\
Piscataway, NJ - 08854
USA. \\
\texttt{praneeth@scarletmail.rutgers.edu},
\and Chetan Tonde\normalfont\textsuperscript{*}, Ahmed Elgammal \\
Department of Computer Science,
Rutgers University,\\
Piscataway, NJ - 08854
USA. \\
 \texttt{\{cjtonde,elgammal\}@cs.rutgers.edu} \\
}

\begin{document}
\maketitle
\begin{abstract}
In our work, we propose a novel formulation for supervised dimensionality reduction based on a nonlinear dependency criterion called Statistical Distance Correlation, \citep{Szekely:2007jx}. We propose an objective which is free of distributional assumptions on regression variables and regression model assumptions. Our proposed formulation is based on learning a low-dimensional feature representation $\mathbf{z}$, which maximizes the squared sum of Distance Correlations between low dimensional features $\mathbf{z}$ and response $y$, and also between features $\mathbf{z}$ and covariates $\mathbf{x}$. We propose a novel algorithm to optimize our proposed objective using the Generalized Minimization Maximizaiton method of \citet{parizi2015generalized}. We show superior empirical results on multiple datasets proving the effectiveness of our proposed approach over several relevant state-of-the-art supervised dimensionality reduction methods.
\end{abstract}

\section{Introduction}
\label{sec:intro}
Rapid developments of imaging technology, microarray data analysis, computer vision, neuroimaging, hyperspectral data analysis and many other applications call for the analysis of high-dimensional data. The problem of supervised dimensionality reduction is concerned with finding a low-dimensional representation of data such that this representation can be effectively used in a supervised learning task. Such representations help in providing a meaningful interpretation and visualization of the data, and also help to prevent overfitting when the number of dimensions greatly exceeds the number of samples, thus working as a form of regularization. In this paper we focus on supervised dimensionality reduction in the regression setting, where we consider the problem of predicting a univariate response $y_i \in \mathbb{R}$ from a vector of continuous covariates $\mathbf{x}_i \in \mathbb{R}^p$, for $i=1$ to $n$. 

Sliced Inverse Regression (SIR) of \citet{sir, sir1,sir2} is one of the earliest developed supervised dimensionality reduction techniques and is a seminal work that introduced the concept of a central subspace that we now describe. This technique aims to find a subspace given by the column space of a $p \times d$ matrix $\textbf{B}$ with $d << p$ such that  $\mathbf{y} \indep \mathbf{X}|\mathbf{B}^T\mathbf{X}$ 
where $\indep$ indicates statistical independence. Under mild conditions the intersection of all such dimension reducing subspaces is itself a dimension reducing subspace, and is called the central subspace \citep{cook}. SIR aims to estimate this central subspace. Sliced Average Variance Estimation (SAVE) of \citet{save1} and \citet{save2} is another early method that can be used to estimate the central subspace. SIR uses a sample version of the first conditional moment $\Exp{\mathbf{X}\mid Y}$ to construct an estimator of this subspace and SAVE uses the sample first and second conditional moments to estimate it. Likelihood Acquired Directions (LAD) of \citet{lad} is a technique that obtains the maximum likelihood estimator of the central subspace under assumptions of conditional normality of the predictors given the response. Like LAD, methods SIR and SAVE rely on elliptical distributional assumptions like Gaussianity of the data.

More recent methods that do not require any distributional assumptions on the marginal distribution of $\mathbf{x}$ or on the conditional distribution of $y$. The authors of Gradient Based Kernel Dimension Reduction (gKDR), \citet{gKDR}, use an equivalent formulation of the conditional independence relation  $\mathbf{y} \indep \mathbf{X}|\mathbf{B}^T\mathbf{X}$  using conditional cross-covariance operators and aim to find a $\mathbf{B}$ that maximizes the mutual information $I(\mathbf{B^TX},\mathbf{y})$. In this work, the authors estimate the conditional cross-covariance operators by using Gaussian kernels. gKDR instead uses kernels only to provide equivalent characterizations of conditional independence using sample estimators of cross-covariance operators.

Sufficient Component Analysis (SCA) of \citet{sca} is another technique where the $\mathbf{B}$ is also learnt using a dependence criterion. SCA aims to maximize the least-squares mutual information given by $SMI(Z,Y)=\frac{1}{2}\int\int (\frac{p_{zy}(z,y)}{p_{z}(z)p_{y}(y)}-1)^2dzdy$ between the projected features $\mathbf{Z} =\mathbf{B}^T\mathbf{X}$ and the response. This is done under orthonormal constraints over $\mathbf{B}$, and the optimal solution is found by approximating $\frac{p_{zy}(z,y)}{p_{z}(z)p_{y}(y)}$ using method of density ratio estimation \citep{densityRatio1, densityRatio2}, and also an analytical closed form solution for the minima is obtained. In \cite{lsdr} (LSDR), the authors optimize this objective using a natural gradient based iterative solution on the Steifel manifold $\mathbb{S}_{d}^{m}(\mathbb{R})$ via a line search along the geodesic in the direction of the natural gradient \citep{manifOpt1, manifOpt2}. 

Our contribution in this paper is as follows: We propose a new formulation for supervised dimensionality reduction that is based on a dependency criterion called Distance Correlation, \citep{dcor}. This setup is free of distributional, as well as regression model assumptions. The novelty in our formulation is that we do not restrict the transformation from $\mathbf{x}$ to $\mathbf{z}$ to be linear, as in case many of the above techniques. To further add, other works of \citet{dcorSubset1,dcorSubset2,dcorSubset3} have used Distance Correlation as a criterion for feature selection in a regression setting. In our work, we show benefits Distance Correlation as a criterion for supervised low-dimensional feature learning. 

In our work we use the following notation: The spectral radius of a matrix $\mathbf{M}$ is denoted by $\lambda_{max}(\mathbf{M})$, $i^{th}$ eigenvalue by $\lambda_i(\mathbf{M})$, and $i^{th}$ generalized eigenvalue $\mathbf{Ax}=\lambda_i\mathbf{Bx}$ by $\lambda_i(\mathbf{A}, \mathbf{B})$. Moreover, $\lambda_{max}(M)$ ($\lambda_{max}(\mathbf{A}, \mathbf{B})$), and $\lambda_{max}(M)$ ($\lambda_{min}(\mathbf{A}, \mathbf{B})$) respectively, the maximum and minimum eigenvalues (generalized eigenvalues) of  matrices $\mathbf{M}$, $\mathbf{A}$ and $\mathbf{B}$. We use the usual partial ordering for symmetric matrices: $\mathbf{A} \succeq \mathbf{B}$ means $\mathbf{A}-\mathbf{B}$ is positive semidefinite; similarly for the relationships $\succeq, \prec, \succ$. The norm $\Norm{\cdot}$ will be either the Euclidean norm for vectors or the norm that it induces for matrices, unless otherwise specified.

\section{Distance Correlation}
\label{sec:dc_definition}
Distance Correlation introduced by \citet{dcor} and \citet{dcor2,dcor3,dcor4} is a measure nonlinear dependencies between random vectors of arbitrary dimensions. We describe below $\alpha$-distance covariance which is an extended version of standard distance covariance for $\alpha=1$.
\begin{definition}{Distance Covariance \citep{Szekely:2007jx}, $\alpha$-dCov}:
	Distance covariance between random variables $\mathbf{x} \in \mathbb{R}^d$ and $\mathbf{y} \in \mathbb{R}^m$ with finite first moments is a nonnegative number given by
	\[
		\mathbb{\nu}^2(\mathbf{x},\mathbf{y})=\int_{\mathbb{R}^{d+m}}|f_{\mathbf{x},\mathbf{y}}(t,s)-f_\mathbf{x}(t)f_\mathbf{y}(s)|^2 w(t,s)dtds
	\]
where $f_\mathbf{x},f_\mathbf{y}$ are characteristic functions of $\mathbf{x},\mathbf{y}$, $f_{\mathbf{x},\mathbf{y}}$ is the joint characteristic function, and $w(t, s)$ is a weight function defined as $w(t, s) = (C(p,\alpha) C(q,\alpha) |t|^{\alpha+p}_p |s|_q^{\alpha+q})^{-1}$ with $C(d,\alpha) = \frac{2\pi^{d/2}\Gamma(1-\alpha/2)}{\alpha2^{\alpha}\Gamma((\alpha + d)/2)}$.
\end{definition}
The distance covariance is zero if and only if random variables $\mathbf{x}$ and $\mathbf{y}$ are independent. From above definition of distance covariance, we have the following expression for Distance Correlation.
\begin{definition}{Distance Correlation \citep{Szekely:2007jx} ($\alpha$-dCorr)}: The squared Distance Correlation between random variables $\mathbf{x} \in \mathbb{R}^d$ and $\mathbf{y} \in \mathbb{R}^m$ with finite first moments is a nonnegative number defined as
	\[
	\rho^2(\mathbf{x},\mathbf{y})
	    =\left\{ \begin{array}{cc}
				 \frac{\mathbb{\nu}^2(\mathbf{x},\mathbf{y})}{\sqrt{\mathbb{\nu}^2(\mathbf{x},\mathbf{x})\mathbb{\nu}^2(\mathbf{y},\mathbf{y})}}, 
				 & \mathbb{\nu}^2(\mathbf{x},\mathbf{x})\mathbb{\nu}^2(\mathbf{y},\mathbf{y})>0.\\
	        0, & \mathbb{\nu}^2(\mathbf{x},\mathbf{x})\mathbb{\nu}^2(\mathbf{y},\mathbf{y})=0.
	        \end{array} \right.
	\]
\end{definition}
The Distance Correlation defined above has the following interesting properties; 1) ${\rho}^2(\mathbf{x},\mathbf{x})$	 is defined for arbitrary dimensions of $\mathbf{x}$ and $\mathbf{y}$, 2) ${\rho}^2(\mathbf{x},\mathbf{y})=0$ if and only if $\mathbf{x}$ and $\mathbf{y}$ are independent, and 3) ${\rho}^2(\mathbf{x},\mathbf{y})$ satisfies the relation $0 \leq \rho^2(\mathbf{x},\mathbf{y}) \leq 1$. In our work, we use $\alpha$-Distance Covariance with $\alpha=2$ and in the following paper for simplicity just refer to it as Distance Correlation. 
 
 We define sample version of distance covariance given samples $\{ (\mathbf{x}_k,\mathbf{y}_k) | k = 1,2,\ldots, n \}$ sampled i.i.d. from joint distribution of random vectors $\mathbf{x} \in \mathbb{R}^d$ and $\mathbf{y} \in \mathbb{R}^m$. To do so, we define two squared Euclidean distance matrices $\mathbf{E}_\mathbf{X}$ and $\mathbf{E}_\mathbf{Y}$,  where each entry $[\mathbf{E}_\mathbf{X}]_{k,l} = \NormS{\mathbf{x}_k-\mathbf{x}_l}$ and $[\mathbf{E}_\mathbf{Y}]_{k,l} = \NormS{\mathbf{y}_k-\mathbf{y}_l}$ with $k,l \in \{ 1,2,\ldots, n\}$. These squared distance matrices are when double-centered, by making their row and column sums zero, and are denoted as $\widehat{\mathbf{E}}_{\mathbf{X}}, \widehat{\mathbf{Q}}_{\mathbf{X}}$, respectively. So given a double-centering matrix $\mathbf{J}=\mathbf{I}-\frac{1}{n}\mathbf{1}\mathbf{1}^T$, we have $\widehat{\mathbf{E}}_\mathbf{X}=\mathbf{J}\mathbf{E}_\mathbf{X}\mathbf{J}$ and $\widehat{\mathbf{E}}_\mathbf{Y}=\mathbf{J}\mathbf{E}_\mathbf{Y}\mathbf{J}$. Hence sample distance correlation (for $\alpha=2$) is defined as follows.

\begin{definition}{Sample Distance Correlation \citep{Szekely:2007jx}}:
Given i.i.d samples $\mathcal{X} \times \mathcal{Y} = \{ (\mathbf{x}_k,\mathbf{y}_k) | k = 1,2,3,\ldots, n\}$ and corresponding double centered Euclidean distance matrices $\widehat{\mathbf{E}}_\mathbf{X}$ and $\widehat{\mathbf{E}}_\mathbf{Y}$, then the squared sample distance correlation is defined as,
\[
    \hat{\mathbb{\nu}}^2(\mathbf{X},\mathbf{Y})=\frac{1}{n^2}\sum_{k,l=1}^{n}[\widehat{\mathbf{E}}_\mathbf{X}]_{k,l}[\widehat{\mathbf{E}}_\mathbf{Y}]_{k,l},
\] and equivalently sample distance correlation is given by
\[
	\hat{\rho}^2(\mathbf{X},\mathbf{Y})
	= \left\{ \begin{array}{cc}
    	\frac{\mathbf{\hat{\nu}}^2(\mathbf{X},\mathbf{Y})}{\sqrt{\mathbf{\hat{\nu}}^2(\mathbf{X},\mathbf{X})\mathbf{\hat{\nu}}^2(\mathbf{Y},\mathbf{Y})}}, & \mathbf{\hat{\nu}}^2(\mathbf{X},\mathbf{X})\mathbf{\hat{\nu}}^2(\mathbf{Y},\mathbf{Y})>0. \\
    	0, & \mathbf{\hat{\nu}}^2(\mathbf{X},\mathbf{X})\mathbf{\hat{\nu}}^2(\mathbf{Y},\mathbf{Y})=0.
	\end{array}
	\right.
\].	
\end{definition}

\section{Laplacian Formulation of Sample Distance Correlation}
\label{sec:laplacian_formulation}
 In this section, we propose a Laplacian formulation of sample distance covariance, and sample distance correlation, which we later use to propose our objective used for supervised dimensionality reduction (SDR).
 
 A graph Laplacian version of sample distance correlation can be obtained as follows,
\begin{lemma}
Given matrices of squared Euclidean distances $\mathbf{E}_\mathbf{X}$ and $\mathbf{E}_\mathbf{Y}$, and Laplacians $\mathbf{L}_\mathbf{X}$ and $\mathbf{L}_\mathbf{Y}$ formed over adjacency matrics $\widehat{\mathbf{E}}_\mathbf{X}$ and $\widehat{\mathbf{E}}_\mathbf{Y}$, the square of sample distance correlation $\hat{\rho}^2(\mathbf{X},\mathbf{Y})$ is given by
\begin{equation}
    \hat{\rho}^2(\mathbf{X},\mathbf{Y}) = \frac{
        \Trace{\mathbf{X}^T\mathbf{L}_\mathbf{Y}\mathbf{X}}
        }{\sqrt{
            \Trace{\mathbf{Y}^T\mathbf{L}_{\mathbf{Y}}\mathbf{Y}}
            \Trace{\mathbf{X}^T\mathbf{L}_\mathbf{X}\mathbf{X}}
        }}.
    \label{proofEqn}
\end{equation}
\end{lemma}
\begin{proof}
Given matrices $\widehat{\mathbf{E}}_\mathbf{X}$, $\widehat{\mathbf{E}}_\mathbf{Y}$, and column centered matrices $\widetilde{\mathbf{X}}$, $\widetilde{\mathbf{Y}}$, from result of \citet{torgerson1952multidimensional} we have that $\widehat{\mathbf{E}}_\mathbf{X}=-2\widetilde{\mathbf{X}}\widetilde{\mathbf{X}}^T$ and $\widehat{\mathbf{E}}_\mathbf{Y}=-2\widetilde{\mathbf{Y}}\widetilde{\mathbf{Y}}^T$. In the problem of multidimensional scaling (MDS) \citep{borg2005modern}, we know for a given adjacency matrix say  $\mathbf{W}$ and a Laplacian matrix $\mathbf{L}$,
\begin{equation} 
    \Trace{\mathbf{X}^T\mathbf{L}\mathbf{X}} = \frac{1}{2} \sum_{i,j}[\mathbf{W}]_{ij}[\mathbf{E}_\mathbf{X}]_{i,j}.
\end{equation}
Now for the Laplacian $\mathbf{L}=\mathbf{L}_{\mathbf{X}}$ and adjacency matrix $\mathbf{W}=\widehat{\mathbf{E}}_\mathbf{Y}$ we can represent $\Trace{\mathbf{X}^T\mathbf{L}_{\mathbf{Y}}\mathbf{X}}$ in terms of $\widehat{\mathbf{E}}_\mathbf{Y}$ as follows,
\begin{align*}
    \Trace{\mathbf{X}^T\mathbf{L}_\mathbf{Y}\mathbf{X}}
    = &\frac{1}{2}\sum_{i,j=1}^{n}[\widehat{\mathbf{E}}_\mathbf{Y}]_{i,j}[\mathbf{E}_\mathbf{X}]_{i,j}.
\end{align*}
From the fact $[\mathbf{E}_\mathbf{X}]_{i,j}=(\braket{\widetilde{\mathbf{x}}_i,\widetilde{\mathbf{x}}_i}+\braket{\widetilde{\mathbf{x}}_j,\widetilde{\mathbf{x}}_j}-2\braket{\widetilde{\mathbf{x}}_i,\widetilde{\mathbf{x}}_j})$, and also $\widehat{\mathbf{E}}_\mathbf{X}=-2\widetilde{\mathbf{X}}\widetilde{\mathbf{X}}^T$ we get
\begin{align*}
    \Trace{\mathbf{X}^T\mathbf{L}_\mathbf{Y}\mathbf{X}}
    &=-\frac{1}{4}\sum_{i,j=1}^{n}[\widehat{\mathbf{E}}_\mathbf{Y}]_{i,j}([\widehat{\mathbf{E}}_\mathbf{X}]_{i,i}+[\widehat{\mathbf{E}}_\mathbf{X}]_{j,j}-2[\widehat{\mathbf{E}}_\mathbf{X}]_{i,j}) \\
    =&\frac{1}{2}\sum_{i,j}[\widehat{\mathbf{E}}_\mathbf{X}]_{i,j}[\widehat{\mathbf{E}}_\mathbf{Y}]_{i,j}
    -\frac{1}{4}\sum_{j}^{n}[\widehat{\mathbf{E}}_\mathbf{X}]_{j,j}\sum_{i}^{n}[\widehat{\mathbf{E}}_\mathbf{Y}]_{i,j} 
    -\frac{1}{4}\sum_{i}^{n}[\widehat{\mathbf{E}}_{X}]_{i,i}\sum_{j}^{n}[\widehat{\mathbf{E}}_\mathbf{Y}]_{i,j} \nonumber 
\end{align*}
Since $\widehat{\mathbf{E}}_\mathbf{X}$ and $\widehat{\mathbf{E}}_\mathbf{Y}$ are double centered matrices $\sum_{i=1}^{n}[\widehat{\mathbf{E}}_\mathbf{Y}]_{i,j} = \sum_{j=1}^{n}[\widehat{\mathbf{E}}_\mathbf{Y}]_{i,j} = 0$ it follows that
\begin{align*}
    \Trace{\mathbf{X}^T\mathbf{L}_\mathbf{Y}\mathbf{X}}=&\frac{1}{2}\sum_{i,j}[\widehat{\mathbf{E}}_\mathbf{X}]_{i,j}[\widehat{\mathbf{E}}_\mathbf{Y}]_{i,j}. \nonumber 
\end{align*}
It also follows that
\[
    \hat{\nu ^2}(\mathbf{X},\mathbf{Y})
    =\frac{1}{n^2}\sum_{i,j=1}^{n}[\widehat{\mathbf{E}}_\mathbf{Y}]_{i,j}[\mathbf{E}_\mathbf{X}]_{i,j} =\frac{2}{n^2}\Trace{\mathbf{X}^T\mathbf{L}_\mathbf{Y}\mathbf{X}}
\]
Similarly, we can express the sample distance covariance using Laplacians $\mathbf{L}_{\mathbf{X}}$ and $\mathbf{L}_{\mathbf{Y}}$ as 
\begin{align*}
    \hat{\nu ^2}(\mathbf{X},\mathbf{Y})
    =\left( \frac{2}{n^2} \right)\Trace{\mathbf{X}^T\mathbf{L}_\mathbf{Y}\mathbf{X}}
    =\left( \frac{2}{n^2} \right)\Trace{\mathbf{Y}^T\mathbf{L}_\mathbf{X}\mathbf{Y}}.
\end{align*}
The sample distance variances can be expressed as 
$\hat{\nu}^2(\mathbf{X},\mathbf{X})=\left(\frac{2}{n^2}\right) \Trace{\mathbf{X}^T\mathbf{L}_\mathbf{X}\mathbf{X}}$ and  
$\hat{\nu} ^2(\mathbf{Y},\mathbf{Y}) = \left(\frac{2}{n^2}\right)\Trace{\mathbf{Y}^T\mathbf{L}_\mathbf{Y}\mathbf{Y}}$ substituting back into expression of sample distance correlation above we get Equation~\ref{proofEqn}.
\end{proof}

\section{Framework}
\label{sec:problem_formulation}
\subsection{Problem Statement}
The goal in supervised dimensionality reduction (SDR) is to learn a low dimensional representation $\mathbf{Z} \in \mathbb{R}^{n \times p}$ of input features $\mathbf{X} \in \mathbb{R}^{n \times d}$, so as to predict the respone vector $\mathbf{y} \in \mathbb{R}$ from $\mathbf{Z}$. The intuition being that $\mathbf{Z}$ captures all information relevant to predict $\mathbf{y}$. Also, during testing, for out-of-sample prediction, for a new data point $\mathbf{x}^*$, we estimate $\mathbf{z}^*$ assuming that it is predictable from $\mathbf{x}^*$. In our proposed formulation, we use aforementioned Laplacian based sample distance correlation to measure dependencies between variables. We propose maximize dependencies between the low dimensional features $\mathbf{Z}$ and response vector $\mathbf{y}$, and also low dimensional features $\mathbf{Z}$ with input features $\mathbf{X}$.  Our objective is to maximize the sum of squares of these two sample distance correlations which is given by,
\begin{align}
f(\mathbf{Z}) & = 
	\hat{\rho}^2(\mathbf{X},\mathbf{Z}) + \hat{\rho}^2(\mathbf{Z},\mathbf{y}) \\
f(\mathbf{Z}) & = \frac{
			\Trace{\mathbf{Z}^T\mathbf{L}_\mathbf{X}\mathbf{Z}}
		}{\sqrt{
			\Trace{\mathbf{X}^T\mathbf{L}_\mathbf{X}\mathbf{X}}
			\Trace{\mathbf{Z}^T\mathbf{L}_\mathbf{Z}\mathbf{Z}}
		}}
	+ \frac{
			\Trace{\mathbf{Z}^T\mathbf{L}_\mathbf{y}\mathbf{Z}}
		}{\sqrt{
			\Trace{\mathbf{y}^T\mathbf{L}_\mathbf{y}\mathbf{y}}
			\Trace{\mathbf{Z}^T\mathbf{L}_\mathbf{Z}\mathbf{Z}}
		}}.
\end{align}
On simplification we get the following optimization problem which we refer to as \textbf{Problem (P)}.
\begin{align}
	\max_{\mathbf{Z}} \hspace{1cm}
	f(\mathbf{Z}) 
	= \frac{
		\Trace{\mathbf{Z}^T\mathbf{S}_{\mathbf{X},\mathbf{y}}\mathbf{Z}}
	}{\sqrt{
		\Trace{\mathbf{Z}^T\mathbf{L}_{\mathbf{Z}}\mathbf{Z}}
	}} & & \textbf{Problem (P)} \nonumber
\end{align}
where $k_X=\frac{1}{\sqrt{\Trace{\mathbf{X}^T\mathbf{L}_\mathbf{X}\mathbf{X}}}}$,   $k_Y=\frac{1}{\sqrt{\Trace{\mathbf{y}^T\mathbf{L}_\mathbf{y}\mathbf{y}}}}$ are constants, and $\mathbf{S}_{\mathbf{X},\mathbf{y}}=k_X\mathbf{L}_{\mathbf{X}} + k_Y \mathbf{L}_{\mathbf{y}}$.

\subsection{Algorithm}
In the proposed problem (\textbf{Problem (P)}), we observe that numerator of our objective is convex while denominator is non-convex due the presence of a square root and a nonlinear Laplacian term  $\mathbf{L}_\mathbf{Z}$ on $\mathbf{Z}$. Hence, this makes direct optimization of this objective practically infeasible. So to optimize \textbf{Problem (P)}, we present a surrogate objective \textbf{Problem (Q)} which lower bounds our proposed original objective. We maximize this lower bound with respect to $\mathbf{Z}$ and show that optimizing this surrogate objective \textbf{Problem (Q)} (lower bound), also maximizes the proposed objective in \textbf{Problem (P)}. We do so by utlizing the Generalized Minorization-Maximization (G-MM) framework of \citet{parizi2015generalized}. 

The G-MM framework of \citet{parizi2015generalized} is an extension of the well known MM framework of \citet{Lange:2000bv}. It removes the equality constraint between both objectives at every iteration $\mathbf{Z}_k$, except at initialization step $\mathbf{Z}_0$. This allows the use a broader class of surrogates that avoid maximization iterations being trapped at sharp local maxima, and also makes the problem less sensitive to problem initializations.

The surrogate lower bound objective is as follows,
\begin{align*}
	\max_{\mathbf{Z}} \hspace{1cm}
	g(\mathbf{Z} , \mathbf{M}) 
	= \frac{
		\Trace{\mathbf{Z}^T\mathbf{S}_{\mathbf{X},\mathbf{y}}\mathbf{Z}}
	}{
		\Trace{\mathbf{Z}^T\mathbf{L}_{\mathbf{M}}\mathbf{Z}}
	} & & \textbf{Problem (Q)}
\end{align*}
where $\mathbf{M} \in \mathbb{R}^{n \times d}$ belongs to the set of column-centered matrices.

The surrogate problem (\textbf{Problem (Q)}) is convex in both its numerator and denominator for a fixed auxiliary variable $\mathbf{M}$. Theorem~\ref{thm:connectingthm} provides the required justification that under certain conditions, maximizing the surrogate \textbf{Problem (Q)} also maximizes the proposed objective \textbf{Problem (P)} . 

An outline of the strategy for optimization is as follows: 
\begin{enumerate}[a)]
    \item \textbf{Initialize}: Initialize $\mathbf{Z}_{0} = \left[ c\mathbf{J}_d,  \mathbf{0}_{(n-d) \times d}^T \right]^T$, a column-centered matrix where $c=\frac{1}{\sqrt[4]{2(d-1)}}$ and $\mathbf{J}_d \in \mathbb{R}^{d\times d}$ is a centering matrix. This is motivated by statement 1) in proof of Theorem~\ref{thm:connectingthm}. 
    \item \label{enum:strat:optim}\textbf{Optimize}: Maximize the surrogate lower bound $\mathbf{Z}_{k+1} = \arg \max g(\mathbf{Z} , \mathbf{Z}_k)$ (See section \ref{sec:theory}).
    \item \label{enum:strat:rescale}\textbf{Rescaling}: Rescale $\mathbf{Z}_{k+1}\gets\kappa\mathbf{Z}_{k+1}$ such that $\Trace{\mathbf{Z}_{k+1}\mathbf{L}_{\mathbf{Z}_{k+1}}\mathbf{Z}_{k+1}}$ is greater than one. This is motivated by proof of statement 3) of Theorem~\ref{thm:connectingthm}, and also the fact that $g(\mathbf{Z} , \mathbf{M}) =g(\kappa\mathbf{Z} , \mathbf{M})$ and $f(\mathbf{Z}) =f(\kappa\mathbf{Z})$ for any scalar $\kappa$.
    \item Repeat step~\ref{enum:strat:optim} and \ref{enum:strat:rescale} above until convergence.
\end{enumerate}

\begin{theorem}
	Under above strategy, maximizing the surrogate \textbf{Problem Q} also maximizes \textbf{Problem P}.
	\label{thm:connectingthm}
\end{theorem}
\begin{proof}
 For convergence it is enough for us to show the following, \citep{parizi2015generalized}:
\begin{enumerate}
	\item \label{enum:connectingthm:equality} $f(\mathbf{Z}_0) = g(\mathbf{Z}_0, \mathbf{Z}_0)$ for $\mathbf{Z}_{0} = \left[ c\mathbf{J}_d,  \mathbf{0}_{(n-d) \times d}^T \right]^T$ and $c=\frac{1}{\sqrt[4]{2(d-1)}}$,
	\item \label{enum:connectingthm:optim} $g(\mathbf{Z}_{k+1} , \mathbf{Z}_k) \geq g(\mathbf{Z}_{k} , \mathbf{Z}_k)$ and,
	\item \label{enum:connectingthm:max} $f(\mathbf{Z}_{k+1}) \geq g(\mathbf{Z}_{k+1} , \mathbf{Z}_k)$
\end{enumerate}
To prove statement~\ref{enum:connectingthm:equality}, for $\mathbf{Z}_{0} = \left[ c\mathbf{J}_d,  \mathbf{0}_{(n-d) \times d}^T \right]^T$, we observe that $\mathbf{Z}_{0}$ column-centered, $\mathbf{L}_{\mathbf{Z}_{0}} = 2\mathbf{Z}_{0}\mathbf{Z}_{0}^T$ and $\mathbf{Z}_{0}^T\mathbf{Z}_{0} = c^2\mathbf{J}_{d}$. Hence we get $\Trace{\mathbf{Z}_{0}^T\mathbf{Z}_{0}\mathbf{L}_{\mathbf{Z}_{0}}\mathbf{Z}_{0}}=c^4\Trace{2\mathbf{J}_{d}}=c^42(d-1)=1$. This proves the required statement $f(\mathbf{Z}_0) = g(\mathbf{Z}_0, \mathbf{Z}_0)=\Trace{\mathbf{Z}^T_0\mathbf{L}_{\mathbf{Z}_0}\mathbf{Z}^T_0}$. 

Statement~\ref{enum:connectingthm:optim} follows from the optimization $\mathbf{Z}_{k+1} = \arg \max g(\mathbf{Z} , \mathbf{Z}_k)$. To prove statement~\ref{enum:connectingthm:max} we have to show that
\[
	\frac{\Trace{\mathbf{Z}_{k+1}^T\mathbf{S}_{\mathbf{X},\mathbf{y}}\mathbf{Z}_{k+1}}}
	{\sqrt{\Trace{\mathbf{Z}_{k+1}^T\mathbf{L}_{\mathbf{Z}_{k+1}}\mathbf{Z}_{k+1}}}}\geq 
	\frac{\Trace{\mathbf{Z}_{k+1}^T\mathbf{S}_{\mathbf{X},\mathbf{y}}\mathbf{Z}_{k+1}}}
	{\Trace{\mathbf{Z}_{k+1}^T\mathbf{L}_{\mathbf{Z}_k}\mathbf{Z}_{k+1}}}.
\]
Since numerators on both sides are equal, it is enough for us to show that
\[
    \sqrt{\Trace{\mathbf{Z}_{k+1}^T\mathbf{L}_{\mathbf{Z}_{k+1}}\mathbf{Z}_{k+1}}} \leq \Trace{\mathbf{Z}_{k+1}^T\mathbf{L}_{\mathbf{Z}_k}\mathbf{Z}_{k+1}}.
\]
Now from Lemma~\ref{thm:indlemma} we have $\Trace{\mathbf{Z}_{k+1}^T\mathbf{L}_{\mathbf{Z}_{k+1}}\mathbf{Z}_{k+1}} \leq \Trace{\mathbf{Z}_{k+1}^T\mathbf{L}_{\mathbf{Z}_k}\mathbf{Z}_{k+1}}$. It follows from the rescaling step (step c) of the optimization strategy that the left hand side $\Trace{\mathbf{Z}_{t+1}\mathbf{L}_{\mathbf{Z}_{t+1}}\mathbf{Z}_{t+1}}$ is always greater that one, and so taking square root of it implies
	$\sqrt{\Trace{\mathbf{Z}_{t+1}\mathbf{L}_{\mathbf{Z}_{t+1}}\mathbf{Z}_{t+1}}}
	\leq  \Trace{\mathbf{Z}_{t+1}^T\mathbf{L}_{\mathbf{Z}_{t}}\mathbf{Z}_{t+1}}$.
\end{proof}
We summarize all of the above steps in Algorithm~\ref{alg:metadiscomax} below and section~\ref{sec:theory} further describes optimization algorithm to solve \textbf{Problem (Q)} required by it.
\begin{algorithm}[!htp]
	\caption{DISCOMAX}
	\label{alg:metadiscomax}
	\begin{algorithmic}[1]
		\Require  Initialize $\mathbf{Z}_{0} = \left[ c\mathbf{J}_d,  \mathbf{0}_{(n-d) \times d}^T \right]^T$, a column-centered matrix where $c=\frac{1}{\sqrt[4]{2(d-1)}}$, $k \gets 0$
		\Ensure  $\mathbf{Z}^* = \arg \max_{\mathbf{Z}} f(\mathbf{Z})$
		\Repeat
		\State{Solve, 	
			\[ 
				\mathbf{Z}_{k+1} = \arg \max_{\mathbf{Z}} g(\mathbf{Z}, \mathbf{Z}_{k}) \hspace{1cm} \text{\textbf{Problem (Q)}}
			\]}  \label{alg:line:goptim}
		\State {Rescale $\mathbf{Z}_{k+1} \gets \kappa \mathbf{Z}_{k+1}$ such that $\Trace{\mathbf{Z}_{k+1}^T\mathbf{L}_{\mathbf{Z}_{k+1}}\mathbf{Z}_{k+1}} \geq 1$} \label{alg:line:invariant}
		\State $k=k+1$
		\Until{ $\NormS{\mathbf{Z}_{k+1} - \mathbf{Z}_{k}} < \epsilon$ }
		\State $\mathbf{Z}^* = \mathbf{Z}_{k+1}$
		\State \Return $\mathbf{Z}^*$
	\end{algorithmic}
\end{algorithm}

\section{Optimization}
\label{sec:theory}
In this section, we propose a framework for optimizing the surrogate objective $g(\mathbf{Z},\mathbf{M})$, referred to as \textbf{Problem (Q)}, for a fixed $\mathbf{M}=\mathbf{Z}_k$. We observe that for a given value of $\mathbf{M}$, $g(\mathbf{Z},\mathbf{M})$ is a ratio of two convex functions. To solve this, we convert this maximization problem to an equivalent minimization problem $h(\mathbf{Z},\mathbf{M})$, by taking its reciprocal \citep{Schaible:1976}. This allows us to utilize the Quadratic Fractional Programming Problem (QFPP) framework of \citet{Dinkelbach:1967kz} and \citet{zhang2008quadratic} to minimize $h(\mathbf{Z},\mathbf{M})$. We refer to this new minimization problem as \textbf{Problem~(R)}. It is stated below.
\begin{align}
	\min_{\mathbf{Z}} \hspace{1cm}
	h(\mathbf{Z}, \mathbf{M}) 
	= \frac{
		\Trace{
		    \mathbf{Z}^T\mathbf{L}_{\mathbf{M}}\mathbf{Z}
		}
	}{
		\Trace{
		    \mathbf{Z}^T\mathbf{S}_{\mathbf{X},\mathbf{y}}\mathbf{Z}
		}
	} & & \textbf{Problem (R)}
\label{eqn:surrogateminfunction}
\end{align}
where $\mathbf{M}=\mathbf{Z}_k$.

In his seminal work \cite{Dinkelbach:1967kz} and later \citet{zhang2008quadratic} proposed a novel framework to solve constrained QFP problems by converting it to an equivalent parametric optimization problem, by introducing a scalar parameter $\alpha \in \mathbb{R}$. We utilize this equivalence proposed to defined new parametric problem, \textbf{Problem (S)}. The solution involves a search over the scalar parameter $\alpha$ while repeatedly solving  \textbf{Problem (S)} to get the required solution $\mathbf{Z}_{k+1}$. This search process continues until values of $\alpha$ converge.

In a nutshell, \citet{Dinkelbach:1967kz} and \citet{zhang2008quadratic} frameworks suggest the following optimizations are equivalent:
\begin{center}
    \begin{tabular}{ccc}
        \begin{tabular}{c}
        	\toprule
        	\textbf{Problem (R)} \\
        	\midrule
        	$\underset{\mathbf{z} \in \mathbb{R}^d}{ \text{minimize}} $ 
        	$ h(\mathbf{z})  = \frac { f_1(\mathbf{z}) }{ f_2(\mathbf{z}) }  $ \\
        	\bottomrule
        \end{tabular} & 
        & \scalebox{1.5}{$\iff$}
        \begin{tabular}{c}
        	\toprule
        	\textbf{Problem (S)} \\
        	\midrule
        	$\underset{\mathbf{z} \in \mathbb{R}^d}{ \text{minimize}} $ 
        	$ H(\mathbf{z}; \alpha^*) =  f_1(\mathbf{z}) - \alpha^* f_2(\mathbf{z})$ \\
        	for some $\alpha^* \in \mathbb{R}$ \\
        	\bottomrule
        \end{tabular}
    \end{tabular}
\end{center}
where
$ f_i(\mathbf{z}) := \mathbf{z}_i^T\mathbf{A}_i\mathbf{z}-2\mathbf{b}_i\mathbf{z}+c_i$ with $\mathbf{A}_1,\mathbf{A}_2 \in \mathbb{R}^{n \times n}$,
$\mathbf{b}_1,\mathbf{b}_2 \in \mathbf{R}^{n}$, and $c_1,c_2 \in \mathbb{R}$. $\mathbf{A}_1$ and $\mathbf{A}_2$ are symmetric with $f_2(\mathbf{x}) > 0 $ over some $\mathbf{z} \in \mathcal{Z}$. 

To see the equivalence of $h(\mathbf{Z}, \mathbf{M})$ in \textbf{Problem (R)} to $h(\mathbf{z})$ above we observe that: $\mathbf{A}_1 = \mathbf{I}_n \otimes  \mathbf{L}_{\mathbf{M}}$, 
$\mathbf{A}_2 = \mathbf{I}_n \otimes  \mathbf{S}_{\mathbf{X},\mathbf{y}}$, $c_i=c_2=0$, and $\mathbf{b}_1=\mathbf{b}_2=\mathbf{0}$. Also, due to positive definiteness of $\mathbf{A}_i$, $f_i(\mathbf{z})$ is positive\footnote{In case of $\mathbf{A}_i$ is semi-definite we regularize by adding $\mathbf{A}_i+\epsilon\mathbf{I}$ so that $\mathbf{A}_i\succ 0$}, and $f(\mathbf{z}_i)>0$. Using this setup for $h(\mathbf{Z}, \mathbf{M})$ we get,\footnote{$\otimes$ indicates kronecker product. $\vectorize{\mathbf{Z}}$ denotes column vectorization of matrix $\mathbf{Z}$.}
\begin{align}
	\min_{\mathbf{Z}} \hspace{1cm}
	h(\mathbf{Z}, \mathbf{M})
			= \frac { 
				\vectorize{\mathbf{Z}}^T(\mathbf{I}_n \otimes \mathbf{L}_{\mathbf{M}})\vectorize{\mathbf{Z}}
			}{
				\vectorize{\mathbf{Z}}^T(\mathbf{I}_n \otimes \mathbf{S}_{\mathbf{X},\mathbf{y}} ) \vectorize{\mathbf{Z}}
			}
	\label{eqn2}
\end{align}

In subsection~\ref{subsec:gss} we propose a Golden Section Search \citep{kiefer1953sequential} based algorithm (Algorithm~\ref{alg:gsdiscomax}) which utilizes concavity property of $ H(\mathbf{Z}; \alpha)$ with respect to $\alpha$ to locate the best $\alpha^*$. During this search we repeatedly solve \textbf{Problem (S)} starting with an intial interval $0=\alpha_l \leq \alpha \leq  \alpha_u = \lambda_{min}(\mathbf{L}_{\mathbf{M}},\mathbf{S}_{\mathbf{X},\mathbf{y}})$ for a fixed $\mathbf{M}$, then at each step shorten the search interval by moving upper and lower limits closer to each other. We continue until convergence to $\alpha^*$. The choice of the upper limit of $\alpha_u = \lambda_{min}(\mathbf{L}_{\mathbf{M}},\mathbf{S}_{\mathbf{X},\mathbf{y}})$ is motivated by proof of Lemma~\ref{thm:alpha}.

To solve \textbf{Problem (S)} for a given $\alpha$, we propose an iterative algorithm in subsection~\ref{subsec:discomax}  (Algorithm \ref{alg:discomax}). It uses the classical Majorization-Minimization framework of \cite{Lange:2013}.

\subsection{Golden Section Search}
\label{subsec:gss}
\cite{Dinkelbach:1967kz} and \cite{zhang2008quadratic} showed the following properties of the objective\footnote{For a fixed $\mathbf{Z}$ and variable argument $\alpha$ we denote $H(\mathbf{Z}; \alpha)$ as $H(\alpha)$.} $H(\alpha)$ with respect to $\alpha$, for a fixed $\mathbf{Z}$.
\begin{theorem}
	Let $G \colon \mathbb{R} \rightarrow \mathbb{R}$ be defined as \[ G(\alpha)= \min_{\mathbf{Z}} H(\mathbf{Z}; \alpha) = \min_{\mathbf{Z}}\left\{\Trace{\mathbf{Z}^{T}\mathbf{L}_{\mathbf{M}}\mathbf{Z}}-\alpha \Trace{\mathbf{Z}^{T}\mathbf{S}_{\mathbf{X},\mathbf{y}}\mathbf{Z}} \right\}\] as derived from \textbf{Problem (S)}, then following statements hold true.
	\begin{enumerate}
		\item $G$ is continuous at any $\alpha \in \mathbb{R}$.
		\item $G$ is concave over $\alpha \in \mathbb{R}$.
		\item $G(\alpha) = 0$, has a unique solution $\alpha^*$.
	\end{enumerate}
	\label{thm:zhang}
\end{theorem}

Algorithm~\ref{alg:gsdiscomax} exploits the concavity property of $G(\alpha)$ to perform a Golden Section Search over $\alpha$. Subsection~\ref{subsec:discomax} provides an iterative Majorization-Minimization algorithm (Algorithm~\ref{alg:discomax}) to solve this minimization problem \textbf{Problem (S)}.
\begin{algorithm}[ht]
\caption{Golden Section Search for $\alpha \in [\alpha_l, \alpha_u]$ for a fixed $\mathbf{M}=\mathbf{Z}_k$.}
\label{alg:gsdiscomax}
		\begin{algorithmic}[1]
			\Require $\epsilon$, $\eta = \frac{1+\sqrt{5}}{2}$, $\alpha_l=0$,  $\mathbf{S}_{\mathbf{X},\mathbf{y}},\mathbf{L}_{\mathbf{X}}$, $\mathbf{L}_{\mathbf{y}}$, $\mathbf{M} = \mathbf{Z}_k$.
			\Ensure $\mathbf{Z}_{k+1} = \arg\min_{\mathbf{Z}} g(\mathbf{Z},\mathbf{Z}_{k+1})$
			\State $\mathbf{D}_X  \gets \diag(\mathbf{L}_\mathbf{X})$
			\State $\mathbf{L}_{\mathbf{M}} \gets 2\mathbf{M}^T\mathbf{M}$
			\State $\alpha_u \gets \lambda_{max}(\mathbf{L}_{\mathbf{M}},\mathbf{S}_{\mathbf{X},\mathbf{y}})$ (Lemma~\ref{thm:gamma})
			\State $\beta \gets \alpha_u + \eta (\alpha_l-\alpha_u)$
			\State $\delta \gets \alpha_l + \eta (\alpha_u-\alpha_l)$
			\Repeat
			\State $H(\beta) \gets 
				\underset{\mathbf{Z} \in \mathbb{R}^d }{\text{minimize}}
				\left( 
					\Trace{
						\mathbf{Z}^{T}\mathbf{L}_{\mathbf{M}}\mathbf{Z}
						}
					-\beta \Trace{
						\mathbf{Z}^{T}\mathbf{S}_{\mathbf{X},\mathbf{y}}\mathbf{Z}
						}
				\right)$  (\textbf{Problem (S)})
			\State $H(\delta) \gets 
				\underset{\mathbf{Z} \in \mathbb{R}^d }{\text{minimize}}
				\left( 
					\Trace{
						\mathbf{Z}^{T}\mathbf{L}_{\mathbf{M}}\mathbf{Z}
						}
					-\delta \Trace{
						\mathbf{Z}^{T}\mathbf{S}_{\mathbf{X},\mathbf{y}}\mathbf{Z}
						}
				\right)$  (\textbf{Problem (S)})
			\If  {$\left(H(\beta) > H(\delta)\right)$}
				\State $\alpha_u \gets \delta$, $\delta \gets \beta$
				\State $\beta \gets \alpha_u + \eta (\alpha_l-\alpha_u)$
			\Else
				\State $\alpha_l \gets \beta$, $\beta \gets \delta$
				\State $\delta \gets \alpha_l + \eta (\alpha_u-\alpha_l)$
			\EndIf
			\Until {$(|\alpha_u-\alpha_l| < \epsilon)$}
			\State $\alpha^* \gets \frac{\alpha_u+\alpha_u}{2} $
			\State $\mathbf{Z}_{k+1} \gets
			{\arg\min_{\mathbf{Z} \in \mathbb{R}^d }}
				\left(
					\Trace{
						\mathbf{Z}^{T}\mathbf{L}_{\mathbf{M}}\mathbf{Z}
						}
					-\alpha^* \Trace{
				        \mathbf{Z}^{T}\mathbf{S}_{\mathbf{X},\mathbf{y}}\mathbf{Z}
						}
				\right)
			$ (\textbf{Problem (S)})
			\State \Return $\alpha^*$, $\mathbf{Z}_{k+1}$
		\end{algorithmic}
\end{algorithm}

\subsection{Distance Correlation Maximization Algorithm}
\label{subsec:discomax}
Algorithm~\ref{alg:discomax} gives a iterative fixed point algorithm which solves \textbf{Problem (S)}. Theorem~\ref{thm:mmobj} provides a fixed point iterate used to minimize $H(\mathbf{Z}, \alpha)$ with respect to $\mathbf{Z}$ for a given $\alpha$.  The fixed point iterate\footnote{We use the subscript $t$ to indicate fixed point iteration of $\mathbf{Z}_t$.} $\mathbf{Z}_{t+1} = \mathbf{H}\mathbf{Z}_{t}$ minimizes \textbf{Problem (S)} and a monotonic convergence is assured by the Majorization-Minimization result of \citet{Lange:2013}. Theorem~\ref{thm:mmobj} below derives the fixed point iterate used in Algorithm~\ref{alg:discomax}.
\begin{theorem}For a fixed $\gamma^2$ (Lemma~\ref{thm:gamma}), some $\alpha$ (Lemma~\ref{thm:alpha}) and \[ \mathbf{H}=\left(
			\gamma ^2\mathbf{D}_X -\alpha \mathbf{S}_{\mathbf{X},\mathbf{y}}
		\right)^{\dagger}
		(\gamma ^2\mathbf{D}_X-\mathbf{L}_{\mathbf{M}}) \]
	the iterate $\mathbf{Z}_t = \mathbf{H}\mathbf{Z}_{t-1}$ monotonically minimizes the objective,
	\begin{align}
		F(\mathbf{Z}; \alpha) = \Trace{
			\mathbf{Z}^T \mathbf{L_{\mathbf{M}}} \mathbf{Z}
			} 
		-\alpha \Trace{
			\mathbf{Z}^T\mathbf{S}_{\mathbf{X},\mathbf{y}}\mathbf{Z}
			}
	\end{align} 
	\label{thm:mmobj}
\end{theorem}

\begin{proof}
    From Lemma \ref{thm:gamma} we know that, $(\gamma^2\mathbf{D}_{\mathbf{X}} - \mathbf{L}_{\mathbf{M}}) \succeq 0$. Hence the following would hold true for any real matrix $\mathbf{N}$,
    \[
    	\Trace{
    	    (\mathbf{Z}-\mathbf{N})^T
    	    (\gamma^2\mathbf{D}_X-\mathbf{L}_{\mathbf{M}})
    	    (\mathbf{Z}-\mathbf{N})
    	} \geq 0 
    \]
    Rearranging the terms we get the following inequality over $\Trace{\mathbf{Z}^T\mathbf{L}_{\mathbf{M}}\mathbf{Z}}$,
    \begin{align*}
        \Trace{\mathbf{Z}^T\mathbf{L}_{\mathbf{M}}\mathbf{Z}} 
        &+\Trace{
            \mathbf{N}^T(\gamma^2\mathbf{D}_X-\mathbf{L}_{\mathbf{M}})\mathbf{Z}
            }
        -\Trace{
    	    \mathbf{N}^T(\gamma^2\mathbf{D}_X-\mathbf{L}_{\mathbf{M}})\mathbf{N}
    	    } \\
        & \leq \Trace{\mathbf{Z}^T\gamma^2\mathbf{D}_X\mathbf{Z}} - \Trace{\mathbf{Z}^T(\gamma^2(\mathbf{D}_X-\mathbf{L}_{\mathbf{M}})\mathbf{N}} \\
        \Trace{\mathbf{Z}^T\mathbf{L}_{\mathbf{M}}\mathbf{Z}} 
        & \leq \Trace{\mathbf{Z}^T\gamma^2\mathbf{D}_X\mathbf{Z}} 
        - 2 \Trace{\mathbf{Z}^T(\gamma^2\mathbf{D}_X-\mathbf{L}_{\mathbf{M}})\mathbf{N}} + \Trace{
    	    \mathbf{N}^T(\gamma^2\mathbf{D}_X-\mathbf{L}_{\mathbf{M}})\mathbf{N}
    	    } \\
        & = l(\mathbf{Z},\mathbf{N})
    \end{align*}
    If $\mathbf{N}=\mathbf{Z}$ then $l(\mathbf{Z},\mathbf{Z})= \Trace{\mathbf{Z}^T\mathbf{L}_{\mathbf{M}}\mathbf{Z}}$. Hence $l(\mathbf{Z},\mathbf{N})$ majorizes $\Trace{\mathbf{Z}^T\mathbf{L}_{\mathbf{M}}\mathbf{Z}}$.  It also follows that the surrogate function $l(\mathbf{Z},\mathbf{N}) -\alpha \Trace{\mathbf{Z}^T\mathbf{S}_{\mathbf{X},\mathbf{y}}\mathbf{Z}}$ majorizes our desired objective function $H(\mathbf{Z}; \alpha)$. To optimize this surrogate loss we equate its gradient to zero and rearrange the terms to obtain
    \begin{align*}
    	(\gamma^2\mathbf{D}_X - \alpha \mathbf{S}_{\mathbf{X},\mathbf{y}}) \mathbf{Z} 
    	& = ({\gamma^2}\mathbf{D}_X - \mathbf{L}_{\mathbf{M}} )\mathbf{N} \\
    	\mathbf{Z} & = (\gamma^2\mathbf{D}_X - \alpha \mathbf{S}_{\mathbf{X},\mathbf{y}})^{\dagger}
    	({\gamma^2}\mathbf{D}_X - \mathbf{L}_{\mathbf{M}})\mathbf{N},
    \end{align*}
    which gives us the update equation ${\mathbf{Z}}_{t+1} = \mathbf{H}{\mathbf{Z}}_{t}$ where $\mathbf{H}$ is given by,
	\begin{align} 
		\mathbf{H} = 
			(\gamma ^2\mathbf{D}_X -\alpha \mathbf{S}_{\mathbf{X},\mathbf{y}})^{\dagger} 
			(\gamma ^2\mathbf{D}_X -\mathbf{L}_{\mathbf{M}}).
	\end{align}
	Hence it follows from framework of \citet{Lange:2013} that above update equation monotonically minimizes $H(\mathbf{Z}; \alpha)$.
\end{proof}

Algorithm~\ref{alg:discomax} summarizes the steps of an iterative Majorization-Minimization approach to solve \textbf{Problem (S)}.
\begin{algorithm}[!htp]
\caption{Distance Correlation Maximization for a given $\alpha$}
\label{alg:discomax}
		\begin{algorithmic}[1]
			\Require $\gamma^2$ (Theorem \ref{thm:gamma}), $\alpha$, $\mathbf{M}=\mathbf{Z}_k$, $\mathbf{S}_{\mathbf{X},\mathbf{y}}$, $\mathbf{L}_{\mathbf{M}}$, $\mathbf{D}_\mathbf{X}$
			\Ensure $
				H(\mathbf{Z}; \alpha) 
				= \underset{\mathbf{Z} \in \mathbb{R}^d }{\text{minimize}}
				\left( 
					\Trace{
						\mathbf{Z}^{T}\mathbf{L}_{\mathbf{M}}\mathbf{Z}
						}
					-\alpha \Trace{
						\mathbf{Z}^{T}\mathbf{S}_{\mathbf{X},\mathbf{y}}\mathbf{Z}
						}
				\right)
			$
			\State $t \gets 0$
			\State $\mathbf{Z}_t = \mathbf{Z}_k$
			\State $H(\mathbf{Z}_t; \alpha) \gets \left( 
					\Trace{ \mathbf{Z}^{T}_t\mathbf{L}_{\mathbf{M}}\mathbf{Z}_t}
					-\alpha \Trace{\mathbf{Z}^{T}_t\mathbf{S}_{\mathbf{X},\mathbf{y}}\mathbf{Z}_t}
				\right)$
			\State $\mathbf{H} =  \left( \gamma^2\mathbf{D}_{\mathbf{X}}
						-\alpha \mathbf{S}_{\mathbf{X},\mathbf{y}} \right)^{\dagger} 
						\left( 
							\gamma^2\mathbf{D}_{\mathbf{X}} -\mathbf{L}_{\mathbf{M}}
						\right)$
			\Repeat
				\State $\mathbf{Z}_{t+1} = \mathbf{H}\mathbf{Z}_{t}$
				\State $H({\mathbf{Z}_{t+1}}; \alpha) \gets 
				\left( 
					\Trace{ \mathbf{Z}^{T}_t\mathbf{L}_{\mathbf{M}}\mathbf{Z}_t}
					-\alpha \Trace{\mathbf{Z}^{T}_t\mathbf{S}_{\mathbf{X},\mathbf{y}}\mathbf{Z}_t}
				\right)$
				\State $t \gets t +1$
			\Until {$(\vert H({\mathbf{Z}_{t+1}}; \alpha)-H({\mathbf{Z}_t; \alpha)} \vert  < \epsilon)$ \textbf{or} $(t \geq T_{\max})$}
			\State $F(\alpha) \gets H(\mathbf{Z}_{t}; \alpha) $
			\State $\mathbf{Z}^* \gets \mathbf{Z}_{t}$
		    \State \Return $F(\alpha), \mathbf{Z}^*$
		\end{algorithmic}
\end{algorithm}

\section{Experiments}
\label{sec:experiments}
In this section we present experimental results that compare our proposed method with several state-of-the-art supervised dimensionality reduction techniques on a regression task.

\subsection{Methodology}
\label{subsec:methodology}
\begin{wrapfigure}{r}{0.4\textwidth}
    \centering
     \includegraphics[width=0.45\textwidth]{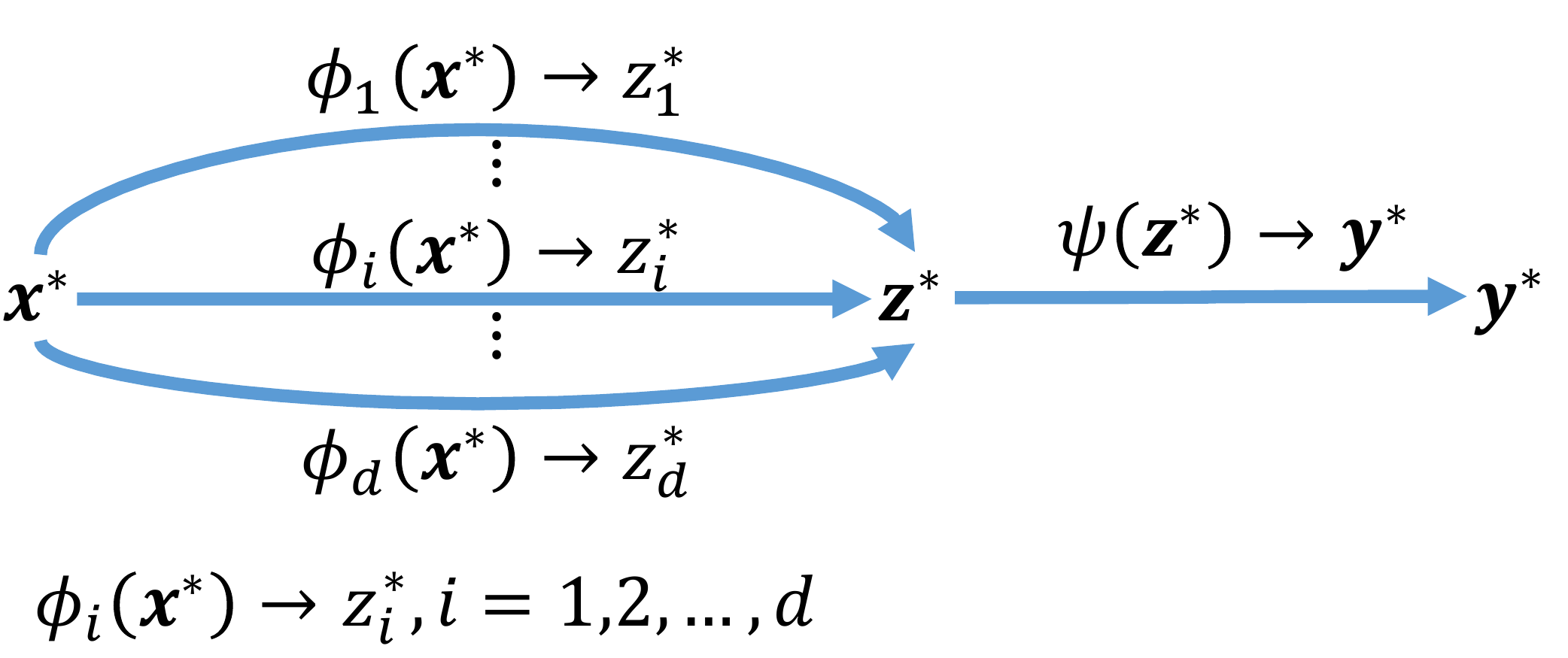}
    \caption{Out-of-Sample prediction}
    \label{figure:block}
\end{wrapfigure}
Methodology we use for our experiments is as follows:
\begin{enumerate}[(i)]
    \item We run our proposed algorithm on the training set $\mathbf{X}_{\texttt{Train}}$ to learn low-dimensional features $\mathbf{Z}_{\texttt{Train}}$.
    \item We learn the map $\psi \colon \mathbf{z} \mapsto y$ using Support Vector Regression on $\mathbf{Z}_{\texttt{Train}}$ and $\mathbf{Y}_{\texttt{Train}}$.
    \item We learn mappings $\phi_i \colon \mathbf{x} \mapsto z_i$, $i=1$ to $d$ for each dimension of $\mathbf{z}$ using Support Vector Regression on $\mathbf{X}_{\texttt{Train}}$ and $\mathbf{Z}_{\texttt{Train}}$.
  \end{enumerate}
 During testing/out-of-sample phase, given a test input $\mathbf{x}^*$, we use maps $\phi_i \colon \mathbf{x} \mapsto z_i$ for $i= 1$ to $d$ and generate $\mathbf{z}^*$. We then utilize maps $\psi \colon \mathbf{z} \mapsto y$ on $\mathbf{z}^*$ to get the predicted response $y^*$. Figure \ref{figure:block} illustrates the testing phase of our methodology.

\subsection{Datasets}
\label{subsec:datasets}
In our results we report the Root Mean Squared (RMS) errors on five datasets from the UCI-Machine Learning Repository \citep{Lichman:2013}  in \Cref{tab:boston,tab:originofmusic,tab:blogger,tab:ctslices,tab:indoorlocalize}. We use the following datasets in our experiments.
\begin{enumerate}[(a)]
    \item \textbf{Boston Housing} \citep{harrison1978hedonic}: This dataset contains information collected by the U.S Census Service concerning housing in the area of Boston Mass. This dataset has been used extensively throughout the vast regression literature to benchmark algorithms. The response variable to be predicted is the median value of owner-occupied homes.
    \item \textbf{Relative Location of Computed Tomography (CT) Slices} \citep{graf20112d}: This dataset consists of 385 features extracted from computed tomography (CT) images. Each CT slice is described by two histograms in polar space that are concatenated to form the final feature vector. The response variable to be predicted is the relative location of an image on the axial axis. The ground truth of responses in this dataset was constructed by manually annotating up to 10 distinct landmarks in each CT Volume with a known location. This response takes values in the range $[0,180]$ where $0$ denotes the top of the head and $180$ denotes the the soles of the feet. 
    \item \textbf{BlogFeedback}
    \citep{buza2014feedback}: This dataset originates from a set of raw HTML documents of blog posts that were crawled and processed. The task associated with this data is to predict the number of comments in the upcoming 24 hours. In order to simulate this situation, the dataset was curated by choosing a base time (in the past) and selecting the blog posts that were published at most 72 hours before the selected base date/time. Then a set of 281 features of the selected blog posts were computed from the information that was available at the basetime. The target is to predict the number of comments that the blog post received in the next 24 hours, relative to the basetime. In the training data, the base times were in the years 2010 and 2011. In the test data the base times were in February and March 2012. 
    \item \textbf{Geographical Origin of Music} \citep{zhou2014predicting}: Instances in this dataset contain audio features extracted from 1059 wave files covering 33 countries/areas. The task associated with the data is to predict the geographical origin of music. The program MARSYAS was used to extract 68 audio features from the wave files. These were appended with 48 chromatic attributes that describe the notes of the scale bringing the total number of features to 116.
    \item \textbf{UJI Indoor Localization} \citep{torres2014ujiindoorloc}: The UJIIndoorLoc is a Multi-Building Multi-Floor indoor localization database that relies on WLAN/WiFi fingerprinting technology. Automatic user localization consists of estimating the position of the user (latitude, longitude and altitude) by using an electronic device, usually a mobile phone. The task is to predict the actual longitude and latitude. The database consists of 19937 training/reference records and 1111 validation/test records. The 529 features contain the WiFi fingerprint, the coordinates where it was taken, and other useful information. Given that this paper focusses on the setting of univariate responses, we only aim to predict the 'Longitude'.
  \end{enumerate}
  
\subsection{Results}
\label{subsec:results}
We perform five-fold cross validation on each of these datasets and report the average Root Mean Square (RMS) error on the hold-out test sets. \Cref{tab:boston,tab:originofmusic,tab:blogger,tab:ctslices,tab:indoorlocalize} present the cross-validated RMS error of our proposed method (DisCoMax), and six other supervised dimensionality reduction techniques namely; LSDR \citep{lsdr}, gKDR \citep{gKDR}, SCA \citep{sca}, LAD \citep{lad}, SAVE \citep{save1} and \citep{save2} and SIR \citep{sir}.

In case of DisCoMax, we use the methodology described in sub-section \ref{subsec:methodology}. For other methods we used in our evaluation, these techniques generate explicit maps to obtain the low-dimensional representations. As in the case of the methodology for DisCoMax, we use these explicit maps and Support Vector Regression (with a RBF kernel) to generate cross-validated RMS errors on the responses.

We fix folds across the seven techniques presented within each of the tables (\Cref{tab:boston,tab:originofmusic,tab:blogger,tab:ctslices,tab:indoorlocalize}). We also compute RMS errors for increasing dimensions $d=3, 5, 7, 9$ and $11$. We note the significant improvement in the predictive performance (smaller error) of DisCoMax learnt features across for all cases with different dimensionality, and also gradual increase performance (smaller error) as we increase dimensionality learnt features. 

For baseline comparison purposes, in case of the Boston Housing dataset, we observe a RMS error of 0.1719 using Support Vector Regression without any dimensionality reduction $(d=13)$. This when compared to DisCoMax RMS errors which ranged between 0.1559 ($d=3$) and 0.1297 ($d=11$) always did worse. We bold errors for DisCoMax for cases where errors were significantly better when compared with their corresponding standard deviations taken into account. 
\begin{table}
    \begin{tabular}{lrrrrr}
    \toprule
	    Method/dimension & 3       & 5        & 7        & 9        & 11       \\
	  \midrule
 	    DisCoMax & \textbf{0.1559} & \textbf{0.1493} & \textbf{0.1327} & \textbf{0.1311} & \textbf{0.1297} \\
	    LSDR \citep{lsdr}	& 0.1978  & 0.1963 & 0.1892 & 0.1886 & 0.1873 \\
	    gKDR \citep{gKDR}	& 0.1997 	& 0.1813 & 0.1762 & 0.1738 & 0.1719 \\
	    SCA \citep{sca}	& 0.1875  & 0.1796 & 0.1708 & 0.1637 & 0.1602 \\
	    LAD \citep{lad} & 0.2019  & 0.1964 & 0.1932 & 0.1917 & 0.1903  \\
	    SAVE \citep{save1} & 0.2045  & 0.1983 & 0.1967 & 0.1952 & 0.1947 \\
	    SIR \citep{sir} & 0.2261  & 0.2193 & 0.2086 & 0.2076 & 0.2068 \\
	  \bottomrule
    \end{tabular}
    \caption {Boston Housing \citep{harrison1978hedonic}: U.S Census Service concerning housing in the area of Boston Mass. To predict median value of owner-occupied homes. Baseline results SVR RMS error 0.1719.}
    \label{tab:boston}
\end{table}

\begin{table}
\begin{tabular}{rlrrrrr}
  \toprule
	 &  Method/d & 3 & 5 & 7 & 9 & 11 \\ 
  \midrule
  		& DisCoMax & \textbf{19.19} & \textbf{18.67} & \textbf{18.14} & \textbf{17.94} & \textbf{17.81} \\ 
	   & LSDR \citep{lsdr} & 23.63 & 22.31 & 22.09 & 21.93 & 21.82 \\ 
	   & gKDR \citep{gKDR} & 24.06 & 23.39 & 22.76 & 22.52 & 22.50 \\ 
	   & SCA \citep{sca} & 23.17 & 24.96 & 24.21 & 23.34 & 23.06 \\
	   & LAD \citep{lad} & 26.74 & 25.57 & 24.39 & 24.26 & 24.20 \\
	   & SAVE \citep{save1} & 28.18 & 27.82 & 27.62 & 27.53 & 27.50 \\ 
	   & SIR \citep{sir} & 29.92 & 29.46 & 29.18 & 28.86 & 28.63 \\
   \bottomrule
\end{tabular}
\caption{Geographical Origin of Music \citep{graf20112d}: The input contains audio features extracted from 1059 wave files covering 33 countries/areas. The task associated with the data is to predict the geographical origin of music.}	
\label{tab:originofmusic}
\end{table}

\begin{table}
    \begin{tabular}{rlrrrrr}
      \toprule
     & Method/d & 3 & 5 & 7 & 9 & 11 \\ 
      \midrule
      & DisCoMax & \textbf{25.82} & \textbf{24.69} & \textbf{24.33} & \textbf{23.90} & \textbf{23.62} \\ 
      & LSDR \citep{lsdr} & 30.36 & 28.16 & 27.39 & 27.24 & 27.18 \\ 
       & gKDR \citep{gKDR} & 29.72 & 27.62 & 27.29 & 26.91 & 26.81 \\
       & SCA \citep{sca} & 28.53 & 27.31 & 26.60 & 26.32 & 26.30 \\ 
       & LAD \citep{lad} & 30.42 & 30.39 & 30.20 & 30.04 & 29.99 \\ 
       & SAVE \citep{save1} & 31.93 & 31.27 & 30.72 & 30.53 & 30.31 \\ 
    & SIR \citep{sir} & 33.63 & 32.65 & 31.39 & 31.16 & 30.83 \\ 
       \bottomrule
    \end{tabular}
    \caption{BlogFeedback \citep{buza2014feedback}: This data contains features computed from raw HTML documents of blog posts. The task associated with this data is to predict the number of comments in the upcoming 24 hours.}	
    \label{tab:blogger}
\end{table}

\begin{table}
	\begin{tabular}{rlrrrrr}
	  \toprule
		 	& Method/d & 3 & 5 & 7 & 9 & 11 \\ 
		\midrule
 		& DisCoMax & \textbf{12.29} & \textbf{11.11} & \textbf{10.19} & \textbf{9.73} & \textbf{9.66} \\ 
		   & LSDR \citep{lsdr} & 14.38 & 13.14 & 12.87 & 12.73 & 12.69 \\
		   & gKDR \citep{gKDR} & 13.65 & 12.86 & 12.67 & 12.35 & 12.05 \\  
		   & SCA \citep{sca} & 14.19 & 13.64 & 12.94 & 12.12 & 11.73 \\ 
		   & LAD \citep{lad} & 17.70 & 17.62 & 17.34 & 17.15 & 16.89 \\ 
		   & SAVE \citep{save1} & 19.32 & 18.74 & 18.62 & 17.76 & 17.21 \\ 
		   & SIR \citep{sir} & 21.53 & 21.23 & 20.97 & 20.77 & 20.64 \\ 
		   \bottomrule
	\end{tabular}
	\caption {Relative location of CT slices  \citep{zhou2014predicting}: Dataset consists of 385 features extracted from CT images. Features are concatenation of two histograms in polar space. The response variable is the relative location of an image on the axial axis. } 	
	\label{tab:ctslices} 
\end{table}
	
\begin{table}
    \begin{tabular}{lrrrrr}
    \toprule
    	 Method/d & 3 & 5 & 7 & 9 & 11 \\
    	 \midrule
    	 DisCoMax & \textbf{12.28} & \textbf{11.10} & \textbf{10.19} & \textbf{9.73} & \textbf{9.65} \\
    	LSDR \citep{lsdr} & 14.38 & 13.14 & 12.86 & 12.73 & 12.69 \\
    	gKDR \citep{gKDR} & 13.65 & 12.86 & 12.67 & 12.34 & 12.05 \\
    	SCA \citep{sca} & 14.18 & 13.63 & 12.94 & 12.12 & 11.73 \\
    	LAD \citep{lad} & 17.69 & 17.62 & 17.34 & 17.15 & 16.89 \\
    	SAVE \citep{save1} & 19.32 & 18.74 & 18.61 & 17.75 & 17.20 \\
    	SIR \citep{sir} & 21.53 & 21.23 & 20.97 & 20.77 & 20.63 \\
    \bottomrule
    \end{tabular}
    \caption {UJI Indoor Localization \citep{torres2014ujiindoorloc}: Multi-Building Multi-Floor indoor localization database. The task is to predict the actual longitude and latitude. The 529 attributes contain the WiFi fingerprint, the coordinates where it was taken. The database consists of around 20ktraining/reference records and 11k validation/test records.}
    \label{tab:indoorlocalize}
\end{table}

\section{Discussion}
\label{sec:discussion}
In this section, we discuss effects of choice of $\alpha$ in the optimization of \textbf{Problem (S)} (Algorithm~\ref{alg:discomax}). We also empirically show optimization of \textbf{Problem (P)} using Algorithm \ref{alg:metadiscomax}, which optimizes a lower bound in \textbf{Problem (Q)}. We use the Boston Housing dataset for our analysis.

Figures \ref{fig:innerAlpha1} and \ref{fig:innerAlpha2} show gradual increase in sample distance correlations $\hat{\rho}(\mathbf{X},\mathbf{Z}_t)$ (Blue) and $\hat{\rho}(\mathbf{Z}_t,\mathbf{y})$ (Red) with respect the number of fixed point $t$ for two different choices of $\alpha = 6\times10^{4}$ and $\alpha = 70\times10^{4}$. We clearly observe that the choice of $\alpha$ has a strong effect on rate of increase/decrease of individual distance correlations $\hat{\rho}^2(\mathbf{X},\mathbf{Z}_t)$ and $\hat{\rho}^2(\mathbf{Z}_t,\mathbf{y})$ as iterations progress. This is because the $\alpha$ value positively weighs the term $\Trace{\mathbf{Z}^T\mathbf{S}_{\mathbf{X},\mathbf{y}}\mathbf{Z}}$ over  $\Trace{\mathbf{Z}^T\mathbf{L}_{\mathbf{M} }\mathbf{Z}}$ in \textbf{Problem (S)}. Figure~\ref{fig:innerSumofsq} shows the rate of change of objective function $f(\mathbf{Z})$ with respect to the fixed point iterations $t$ for two choices of $\alpha$. The figure clearly shows the slower (faster) rate of increase of $f(\mathbf{Z})$ for smaller (larger) $\alpha$.

Figure~\ref{fig:outerMMdcorr} and \ref{fig:outerMMsumofsq} repectively show the overall growth of distance correlations ($\hat{\rho}(\mathbf{X},\mathbf{Z})$, $\hat{\rho}(\mathbf{Z},\mathbf{y})$) and $f(\mathbf{Z})$, with respect to the fixed point iterations ($t$), for $\alpha^*=800\times10^{4}$. We periodically observe a sharp increases in $f(\mathbf{Z})$ and distance correlations after each DisCoMax subproblem of 220 fixed point iterations. The figures show four such G-MM iterations of Algorithm~\ref{alg:metadiscomax}. These sharp increases are due to the resubstitution of $\mathbf{M}=\mathbf{Z}_k$ in Step~\ref{alg:line:goptim} of Algorithm~\ref{alg:metadiscomax}. This clearly shows us that we are able to maximized are original proposed objective in \textbf{Problem (P)}.

\begin{figure}
    \begin{minipage}[b]{0.32\textwidth}
        \centering
        \includegraphics[width=1\textwidth]{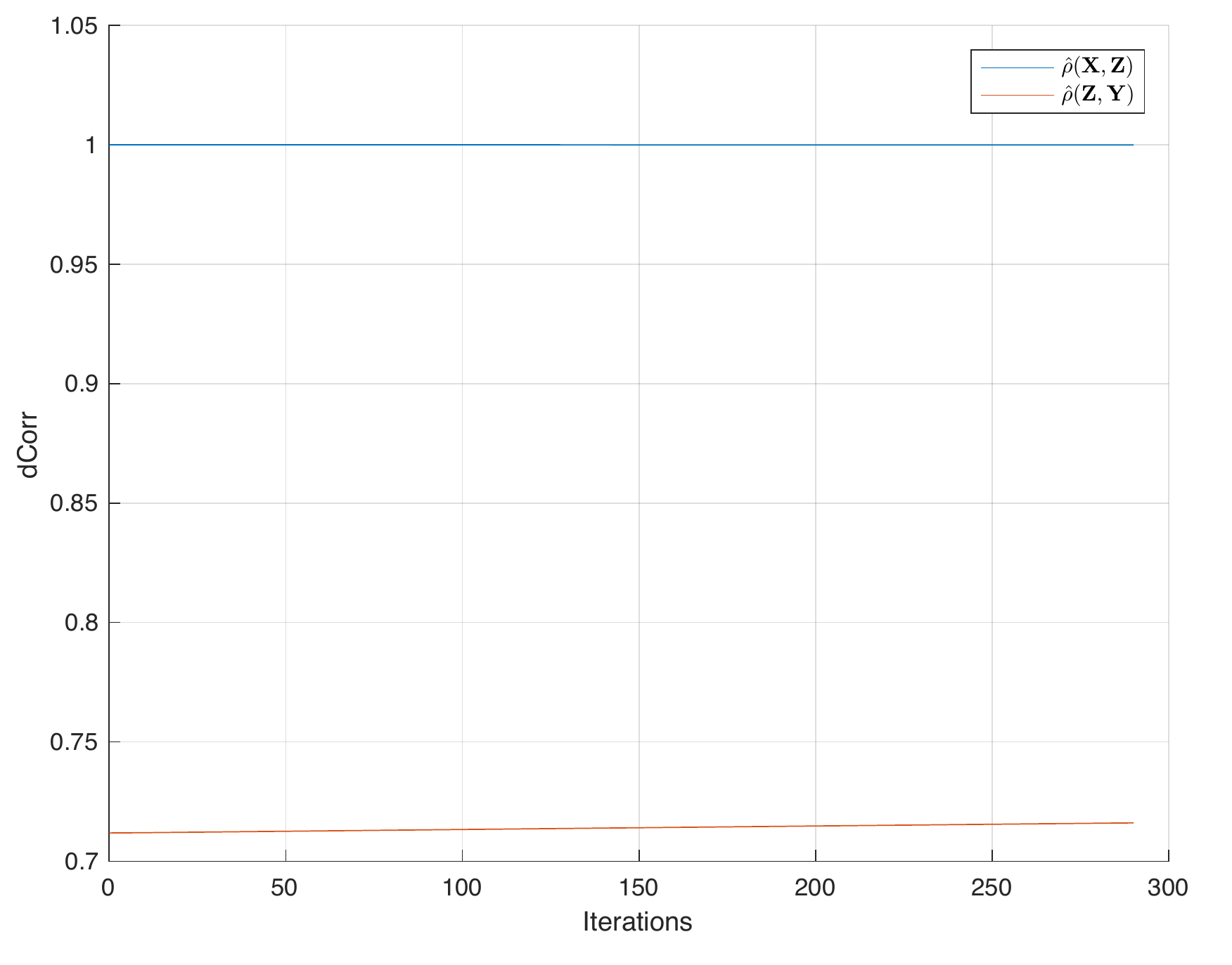}
        \subcaption{$\hat{\rho}(\mathbf{X},\mathbf{Z})$ (Blue) and $\hat{\rho}(\mathbf{Z},\mathbf{y})$ (Red) vs. fixed point iterations ($t$) for $\alpha = 6\times10^{4}$}
        \label{fig:innerAlpha1}
    \end{minipage}
    \begin{minipage}[b]{0.32\textwidth}
        \centering
        \includegraphics[width=1\textwidth]{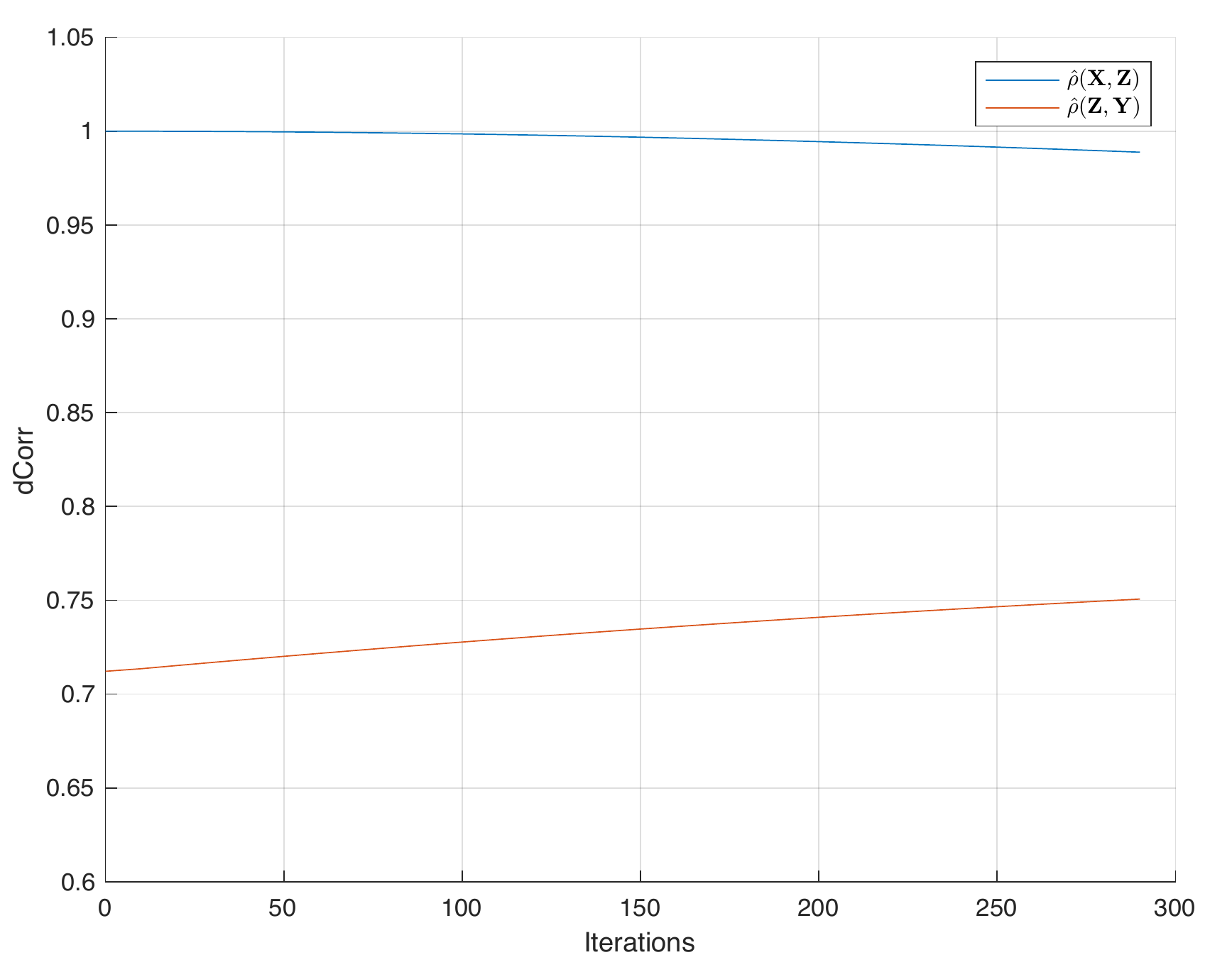} 
        \subcaption{$\hat{\rho}(\mathbf{X},\mathbf{Z})$ (Blue) and $\hat{\rho}(\mathbf{Z},\mathbf{y})$ (Red) vs. fixed point iterations ($t$) for $\alpha=70\times10^{4}$}
        \label{fig:innerAlpha2}
    \end{minipage}
    \begin{minipage}[b]{0.32\textwidth}
        \centering
        \includegraphics[width=1\textwidth]{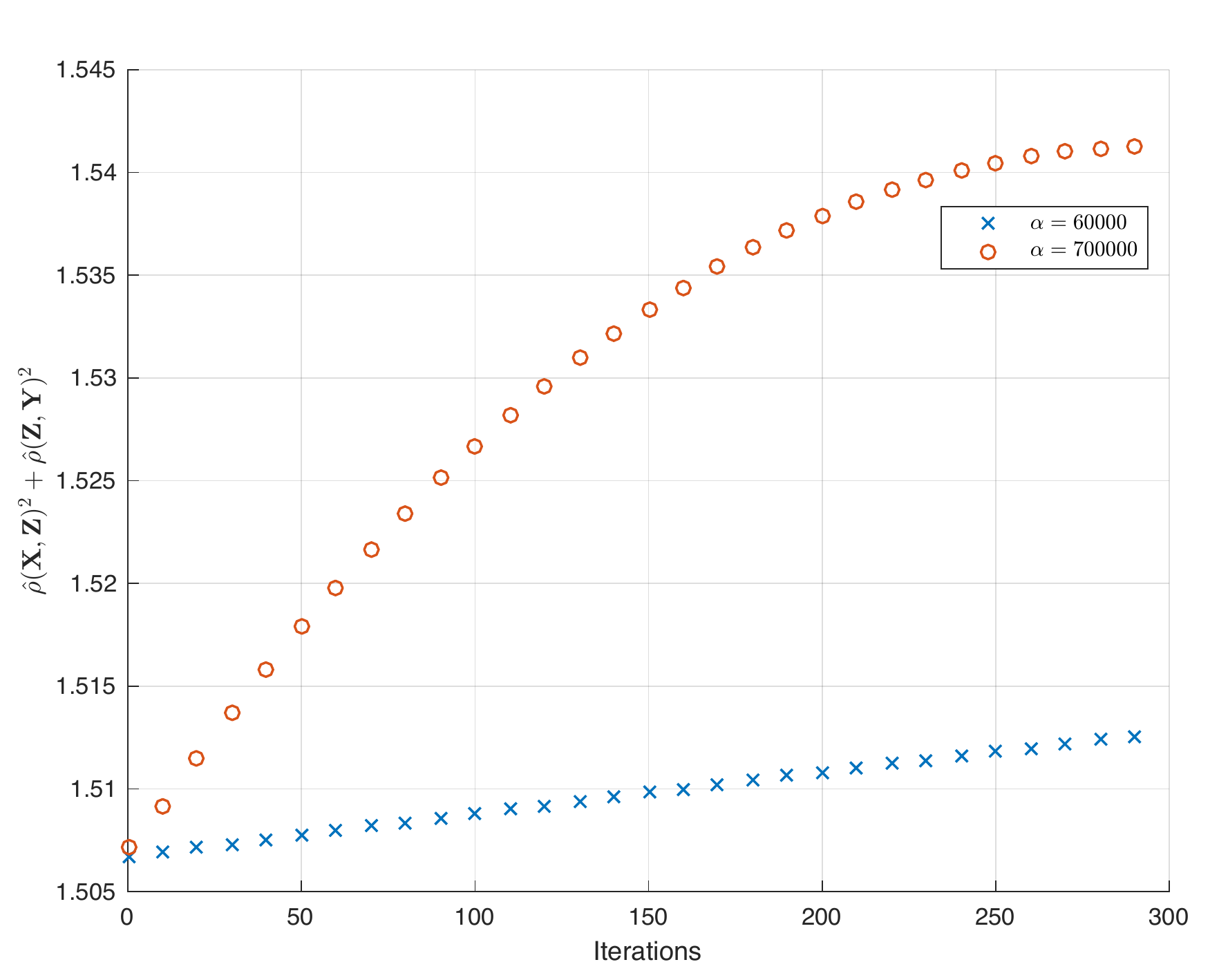} 
        \subcaption{$f(\mathbf{Z})$ vs. fixed point iterations for $\alpha = 6\times10^{4}$ (Red) for $\alpha=70\times10^{4}$ (Blue).}
        \label{fig:innerSumofsq}
    \end{minipage}
    \caption{Effect of $\alpha$ values on growth of the proposed objective in Algorithm~\ref{alg:discomax} the figures show slower (faster) growth of distance correlations for smaller (larger) $\alpha$.}
\end{figure}

\begin{figure}
    \begin{minipage}[b]{0.5\textwidth}
        \centering
        \includegraphics[width=1\textwidth]{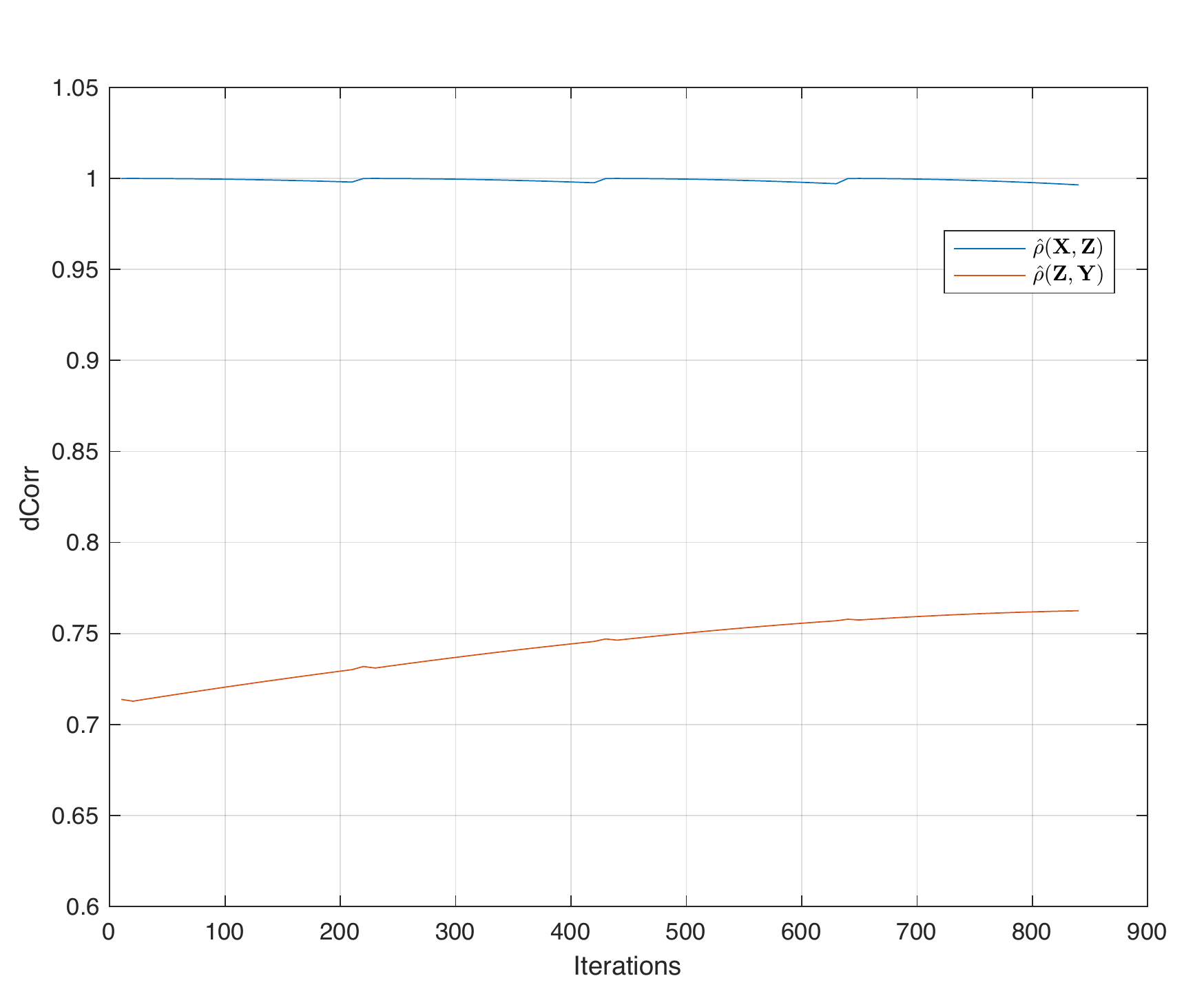}
        \subcaption{$\hat{\rho}(\mathbf{X},\mathbf{Z})$ (Blue) and $\hat{\rho}(\mathbf{Z},\mathbf{y})$ (Red)  vs overall Iterations.}
        \label{fig:outerMMdcorr}
    \end{minipage}
    ~
    \begin{minipage}[b]{0.5\textwidth}
        \centering
        \includegraphics[width=1\textwidth]{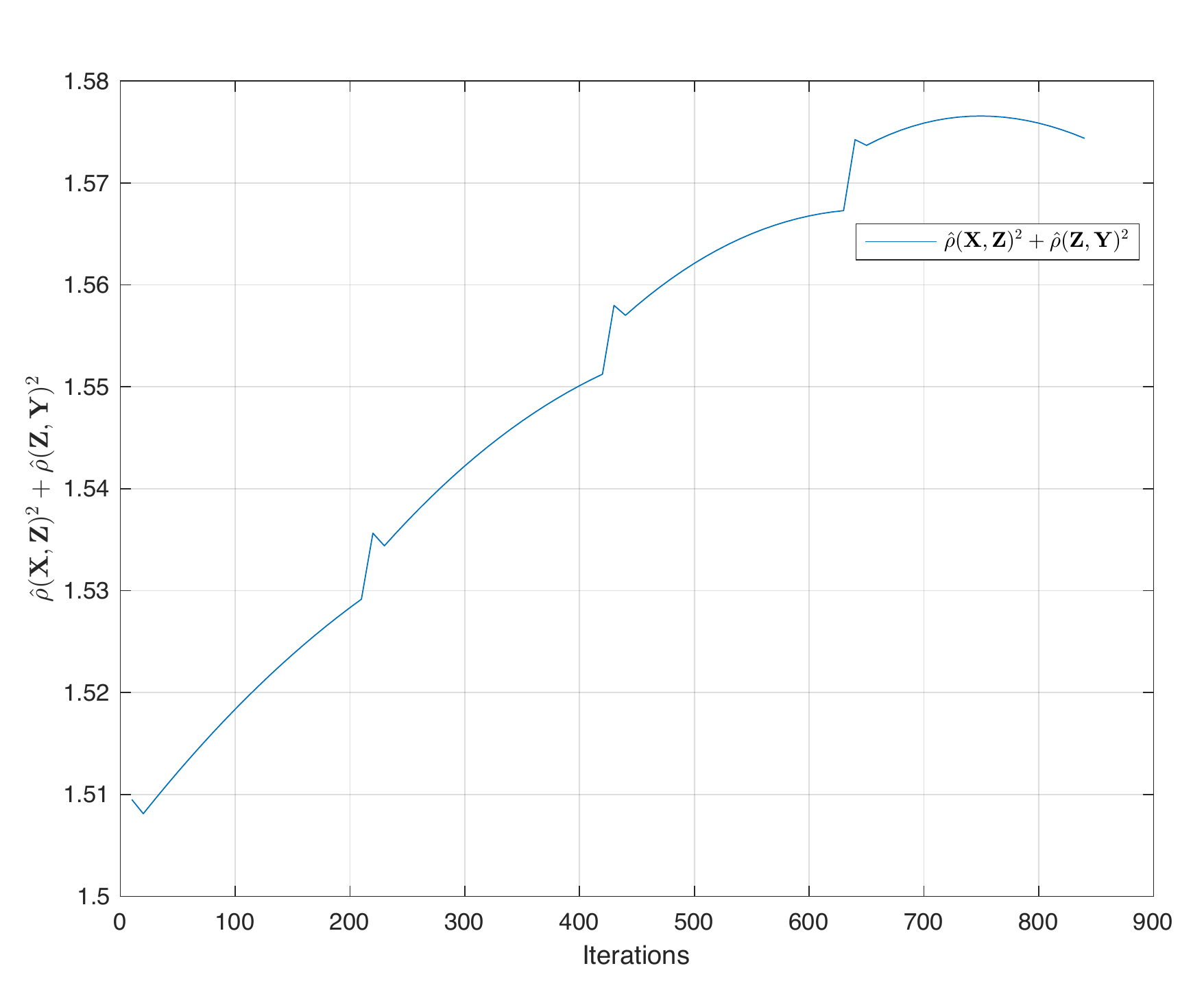}
        \subcaption{$f(\mathbf{Z})=\hat{\rho}(\mathbf{X},\mathbf{Z})^2+\hat{\rho}(\mathbf{Z},\mathbf{y})^2$ vs overall Iterations.}
        \label{fig:outerMMsumofsq}
    \end{minipage}
    \caption{Overall gradual increase in $f(\mathbf{Z})$ (Figure~\ref{fig:outerMMdcorr}) and distance correlations (Figure~\ref{fig:outerMMsumofsq}) for $\alpha^*=800\times10^{4}$. Plots show increase in both for each DisCoMax subproblem of (Algorithm~\ref{alg:discomax}) and four outer G-MM iterations of Algorithm~\ref{alg:metadiscomax}}
    \label{fig:outerMM}
\end{figure}

\section{Conclusion}
\label{sec:conclusion}
In our work, we proposed a novel method to perform supervised dimensionality reduction. Our method aims to maximize an objective based on a statistical measure of dependence called statistical distance correlation. Our proposed method does not necessarily constrain the dimension reduction projection to be linear. We also propose a novel algorithm to optimize our proposed objective using the Generalized Minorization-Maximization approach of \citet{parizi2015generalized}. Finally, we show a superior empirical performance of our method on several regression problems in comparison to existing state-of-the-art methods. 

For future work, we aim to extend our framework to handle multivariate responses $\mathbf{y}\in\mathbb{R}^q$, as distance correlation is applicable to variables with arbitrary dimensions. Our proposed approach is practically applicable on relatively small datasets, as it involves repeatedly solving multiple optimization subproblems. So we aim to to simplyfy this approach so that it is tractable for larger size (several thousands of examples) datasets. In our work, we currently tackle the out-of-sample issue by learning mutiple SVR's, one for each dimension of $\mathbf{z}$, we plan to extend our framework so as to learn explicit out-of-sample mappings from $\mathbf{x}$ to $\mathbf{z}$.

\FloatBarrier
\bibliography{conference}
\bibliographystyle{plainnat}

\newpage
\appendix
\section{Spectral Radius of the Fixed Point Iterate $T(\mathbf{Z}_t)$}
To prove Lemma~\ref{thm:indlemma}, required for proving convergence in Theorem~\ref{thm:connectingthm}, we need to show that the spectral radius $\lambda_{max}(\mathbf{H})<1$. We show this in Theorem~\ref{thm:fixedpointiterate} and proceed to prove it by first by proving two required lemmas below.
\begin{lemma}
	For any choice of $ \gamma^2  > \lambda_{max}(\mathbf{D}_{\mathbf{X}}, \mathbf{L}_{\mathbf{M}})$ and $\mathbf{P} \colon= \left( \gamma^2 \mathbf{D}_{\mathbf{X}} - \mathbf{L}_{\mathbf{M}} \right)$, we have  $\mathbf{P} \succeq 0 $.
	\label{thm:gamma}
\end{lemma}
\begin{proof}
	To show $\mathbf{z}^T(\gamma^2\mathbf{D}_{\mathbf{X}}-\mathbf{L}_{\mathbf{M}})\mathbf{z} \geq 0$ for all $\mathbf{z}$, we require that $\gamma^2 \geq \frac{\mathbf{z}^T\mathbf{L}_{\mathbf{M}}\mathbf{z}}{\mathbf{z}^T\mathbf{D}_{\mathbf{X}}\mathbf{z}}$ for all $\mathbf{z}$. This is always true for all values of $\gamma^2 \geq \lambda_{max}(\mathbf{D}_{\mathbf{X}}, \mathbf{L}_{\mathbf{M}})$.
\end{proof}

\begin{lemma}
	If $0 = \alpha_l \leq \alpha \leq \alpha_u=\lambda_{min}(\mathbf{L}_{\mathbf{M}},\mathbf{S}_{\mathbf{X},\mathbf{y}})$ and $\mathbf{Q}\colon=\left(  \mathbf{L}_{\mathbf{M}} - \alpha \mathbf{S}_{\mathbf{X},\mathbf{y}} \right)$, then we have $\mathbf{Q} \succeq 0 $.
	\label{thm:alpha}
\end{lemma}
\begin{proof}
	To show $\mathbf{z}^T(\mathbf{L}_{\mathbf{M}} - \alpha\mathbf{S}_{\mathbf{X},\mathbf{y}})\mathbf{z} \geq 0$ for all $\mathbf{z}$, we require that $\alpha \leq \frac{
		\mathbf{z}^T\mathbf{L}_{\mathbf{M}} \mathbf{z}
		}{
		\mathbf{z}^T\mathbf{S}_{\mathbf{X},\mathbf{y}} \mathbf{z}
	}$ for all $\mathbf{z}$. 
	This is always true if all values of 
	$\alpha \leq \min_{\mathbf{Z}}\frac{
		\mathbf{z}^T\mathbf{L}_{\mathbf{M}} \mathbf{z}
		}{
		\mathbf{z}^T\mathbf{S}_{\mathbf{X},\mathbf{y}} \mathbf{z}
	 }= \lambda_{min}(\mathbf{L}_{\mathbf{M}},\mathbf{S}_{\mathbf{X},\mathbf{y}})$ which is true by our choice of $\alpha$.
\end{proof}

We now utilize the above to results to prove $\lambda_{max}(\mathbf{H}) \leq 1 $ about the fixed point iterate $\mathbf{Z}_{t+1}=\mathbf{H}\mathbf{Z}_t$.
\begin{theorem}
	For the update equation ${\mathbf{Z}}_{t+1} = \mathbf{H}{\mathbf{Z}}_{t}$ with 
	\[
	    \mathbf{H} = 
	    \left( 
		    \gamma^2 \mathbf{D}_{\mathbf{X}}  - \alpha \mathbf{S}_{\mathbf{X},\mathbf{y}} 
	    \right)^{\dag} 
	    \left( 
	    \gamma^2 \mathbf{D}_{\mathbf{X}} - \mathbf{L}_{\mathbf{M}}  
	    \right),
	\]
	we have $\lambda_{max}(\mathbf{H}) \leq 1$.
	\label{thm:fixedpointiterate}
\end{theorem}

\begin{proof}
The update equation looks as follows
\begin{align*}
	{\mathbf{Z}}_{t+1} = 
	\left( 
		\gamma^2 \mathbf{D}_{\mathbf{X}} - \alpha \mathbf{S}_{\mathbf{X},\mathbf{y}} 
	\right)^{\dagger} 
	\left( 
		\gamma^2 \mathbf{D}_{\mathbf{X}} - \mathbf{L}_{\mathbf{M}}  
	\right) 
	{\mathbf{Z}}_{t}.
\end{align*}
For sake of simplicity assume $\mathbf{P} = \left( \gamma^2 \mathbf{D}_{\mathbf{X}} - \mathbf{L}_\mathbf{M} \right)$ and $\mathbf{Q} = \left(  \mathbf{L}_\mathbf{M} - \alpha \mathbf{S}_{\mathbf{X},\mathbf{y}} \right) $.
\begin{align*}
	{\mathbf{Z}}_{t+1} = 
	\left( 
		\mathbf{P} + \mathbf{Q} 
	\right)^{-1} 
	\mathbf{P} {\mathbf{Z}}_{t}
\end{align*}
Using the Woodbury matrix identity 
$(\mathbf{A} + \mathbf{U}\mathbf{B}\mathbf{V})^{-1} = \mathbf{A}^{-1} - \mathbf{A}^{-1}\mathbf{U}(\mathbf{B}^{-1}+\mathbf{V}\mathbf{A}^{-1}\mathbf{U})^{-1}\mathbf{V}\mathbf{A}^{-1}$, and
setting $\mathbf{U} = \mathbf{I}$ and $\mathbf{V} = \mathbf{I}$, we get, $(\mathbf{A} + \mathbf{B})^{-1} = \mathbf{A}^{-1} - \mathbf{A}^{-1} (\mathbf{B}^{-1}+  \mathbf{A}^{-1})^{-1}\mathbf{A}^{-1}$.
Applying this to the previous equation we get
\begin{align*}
	{\mathbf{Z}}_{t+1}  
	& =  (\mathbf{P}^{-1} - \mathbf{P}^{-1} (\mathbf{P}^{-1} 
	+  \mathbf{Q}^{-1})^{-1} \mathbf{P}^{-1}) 
	\mathbf{P}{\mathbf{Z}}_{t}  =  \mathbf{I} - \mathbf{P}^{-1} (\mathbf{P}^{-1} +  \mathbf{Q}^{-1})^{-1} {\mathbf{Z}}_{t}  \\
	 & =  \mathbf{I} - \mathbf{P}^{-1} 
	 \left( 
	 	(\mathbf{P}^{-1} +  \mathbf{Q}^{-1})^{-1}\mathbf{Q}^{-1} 
	 \right) \mathbf{Q} {\mathbf{Z}}_{t}
\end{align*}
Using the positive definite identity $(\mathbf{P}^{-1} + \mathbf{B}^T\mathbf{Q}^{-1}\mathbf{B})^{-1}
	 \mathbf{B}^T \mathbf{Q}^{-1} 
	= \mathbf{P}\mathbf{B}^T (\mathbf{B}\mathbf{P}\mathbf{B}^T
	+ \mathbf{Q})^{-1}$
for $\mathbf{B} = \mathbf{I}$ we get, $(\mathbf{P}^{-1} + \mathbf{Q}^{-1})^{-1} \mathbf{Q}^{-1} = \mathbf{P}(\mathbf{P}+\mathbf{Q})^{-1}$, which simplifies the term in the brackets as,
\begin{align*}
	{ \mathbf{Z}}_{t+1}  
	& =  \mathbf{I} - \mathbf{P}^{-1} 
	\left( \mathbf{P}(\mathbf{P}+\mathbf{Q})^{-1} \right) 
	\mathbf{Q} {\mathbf{Z}}_{t}
	=  \mathbf{I} - (\mathbf{P}+\mathbf{Q})^{-1}\mathbf{Q} {\mathbf{Z}}_{t}
\end{align*}
If we compare the above equation with a the general update equation from \cite{YinZhangRichard:1999ty}, which is of the form
\begin{align*}
	T({\mathbf{Z}}_{t+1})  =  {\mathbf{Z}}_{t} 
	- \beta({\mathbf{Z}}_{t})\mathbf{B}({\mathbf{Z}}_{t})^{-1}
	\nabla f ({\mathbf{Z}}_{t})
\end{align*}
where $\nabla f(\mathbf{Z}_t)$ is the gradient of the objective function $f(\mathbf{Z})$ we get,
	\begin{align*}
		\beta({\mathbf{Z}}_t) = \frac{1}{2},\hspace{1cm}
		\mathbf{B}({\mathbf{Z}}_t)  = \mathbf{P} + \mathbf{Q},\hspace{1cm}
		\nabla f ({\mathbf{Z}}_t) = 2\mathbf{Q}{\mathbf{Z}_t}
	\end{align*}
Now from Theorem~\ref{thm:gamma} we conclude that $\mathbf{B}({\mathbf{Z}}) \succeq 0 $,
We also check the following condition from \cite{YinZhangRichard:1999ty} that
\begin{align*}
	0 \preceq  \nabla^2f(\mathbf{Z}) \preceq \frac{2\mathbf{B}}{\beta}.
\end{align*}
or equivalently, as in our case $0 \preceq  2\mathbf{Q} \preceq 4 (\mathbf{Q}+ \mathbf{P})$, which is indeed true. Hence it follows that  $\lambda_{max}(T'(\mathbf{Z})) \leq 1$ which implies $\lambda_{max}(\mathbf{H}) \leq 1$.
\end{proof}
We now proceed to show that at end of every $(t+1)$ fixed point iterations we  have $\Trace{\mathbf{Z}_{t+1}^T\mathbf{L}_{\mathbf{Z}_{t+1}}\mathbf{Z}_{t+1}} \leq \Trace{\mathbf{Z}_{t+1}\mathbf{L}_{\mathbf{Z}_0}\mathbf{Z}_{t+1}}$.

\begin{lemma} For fixed point iteration $\mathbf{Z}_{t+1} = \mathbf{H}\mathbf{Z}_t$ for optimization of $\mathbf{Z}_{k+1}=\arg\max_{\mathbf{Z}} g(\mathbf{Z}, \mathbf{Z}_k)$, we have, $\Trace{\mathbf{Z}_{k+1}^T\mathbf{L}_{\mathbf{Z}_{k+1}}\mathbf{Z}_{k+1}} \leq \Trace{\mathbf{Z}_{k+1}\mathbf{L}_{\mathbf{Z}_k}\mathbf{Z}_{k+1}}$.
    \label{thm:indlemma}
\end{lemma}
\begin{proof}
	Laplacian for a weighted adjacency matrix $\mathbf{W}$ (with self loops) is defined as $\mathbf{L} = \mathbf{D} - \mathbf{W}$ where $\mathbf{D}$ is a diagonal degree matrix with diagonal elements $[\mathbf{D}]_{i,i}=\sum_j [\mathbf{W}]_{i,j}$ and zero off-diagonal entries \citep{chung1997lecture}. 
	For adjacency matrix $\widehat{\mathbf{E}}_\mathbf{Z}$ we have $\widehat{\mathbf{E}}_\mathbf{Z}=\mathbf{J}\mathbf{E}_{\mathbf{Z}}\mathbf{J}=-2\widetilde{\mathbf{Z}}\widetilde{\mathbf{Z}}^T$ \citep{torgerson1952multidimensional}. We have Laplacian as $\mathbf{L}_{\mathbf{Z}}=\mathbf{D}_{\mathbf{Z}} - \widehat{\mathbf{E}}_\mathbf{Z}$ with $\mathbf{D}_{\mathbf{Z}}=0$. This gives us for $\mathbf{Z}_{t+1}$ the Laplacian $\mathbf{L}_{\mathbf{Z}_{t+1}} = 2\mathbf{Z}_{t+1}\mathbf{Z}_{t+1}^T$. It also follows from the fact that since we choose our intialization $\mathbf{Z}_0$ as column-centered matrix, and $\mathbf{Z}_{t+1}=\mathbf{H}\mathbf{Z}_t$ are also successively  column-centered for all $t>0$. Hence, $\mathbf{L}_{\mathbf{Z}_{t+1}} =2\widehat{\mathbf{Z}}_{t+1}\widehat{\mathbf{Z}}_{t+1}^T$.
	Now substituting $\mathbf{Z}_{t+1} = \mathbf{H}\mathbf{Z}_{t}$ in  Laplacian equation $\mathbf{L}_{\mathbf{Z}_{t+1}} $ we get,
	\begin{align}
		\mathbf{L}_{\mathbf{Z}_{t+1}}
		= 2(\mathbf{H}\mathbf{Z}_{t})(\mathbf{H}\mathbf{Z}_{t})^T
		= 2\mathbf{H}\mathbf{Z}_{t}\mathbf{Z}_{t}^T\mathbf{H}^T
		= \mathbf{H}\mathbf{L}_{\mathbf{Z}_{t}} \mathbf{H}^T.
		\label{eqn:LZrelation}
	\end{align}
	Substituting above equation into right hand side of the statement to be proved gives us,
	\begin{align*}
		\Trace{\mathbf{Z}_{t+1}^T\mathbf{L}_{\mathbf{Z}_{t+1}}\mathbf{Z}_{t+1}}
		& = \Trace{
			\mathbf{Z}_{t+1}^T\mathbf{H}\mathbf{L}_{\mathbf{Z}_{t}}
			\mathbf{H}^T\mathbf{Z}_{t+1}
			}.
	\end{align*}
	Substituting eigen decomposition of $\mathbf{H} = \mathbf{Q}\mathbf{\Lambda}\mathbf{Q}^T$ where  $\mathbf{\Lambda}$ is a diagonal eigenvalues matrix with values less than one (Theorem~\ref{thm:fixedpointiterate}) we get,
	\begin{align*}
		\Trace{\mathbf{Z}_{t+1}\mathbf{L}_{\mathbf{Z}_{t+1}}\mathbf{Z}_{t+1}}
		& = \Trace{
			\mathbf{Z}_{t+1}^T(\mathbf{Q}\mathbf{\Lambda}\mathbf{Q}^T)
			\mathbf{L}_{\mathbf{Z}_{t}}
			(\mathbf{Q}^T\mathbf{\Lambda}\mathbf{Q})\mathbf{Z}_{t+1}
			}.
	\end{align*}
	For $\mathbf{\Lambda} = \mathbf{I}$ (identity matrix) gives us,
	\begin{align*}
		\Trace{\mathbf{Z}_{t+1}\mathbf{L}_{\mathbf{Z}_{t+1}}\mathbf{Z}_{t+1}}
		& \leq  \Trace{\mathbf{Z}_{t+1}^T(\mathbf{QIQ}^T) 
			\mathbf{L}_{\mathbf{Z}_{t}}
			(\mathbf{Q}^T\mathbf{I}\mathbf{Q})\mathbf{Z}_{t+1}}  \leq  \Trace{\mathbf{Z}_{t+1}^T\mathbf{L}_{\mathbf{Z}_{t}}\mathbf{Z}_{t+1}}.
	\end{align*}
	Repeating the above process until $t=0$ we get $\Trace{\mathbf{Z}_{t+1}\mathbf{L}_{\mathbf{Z}_{t+1}}\mathbf{Z}_{t+1}} \leq  \Trace{\mathbf{Z}_{t+1}^T\mathbf{L}_{\mathbf{Z}_0}\mathbf{Z}_{t+1}}$.
	Now, for the initialisation $\mathbf{Z}_t=\mathbf{Z}_k$ at $t=0$, and given that $\mathbf{Z}_{k+1}=\arg\max_{\mathbf{Z}} g(\mathbf{Z}, \mathbf{Z}_k)$ we have, 
	\[
	    \Trace{\mathbf{Z}_{k+1}\mathbf{L}_{\mathbf{Z}_{k+1}}\mathbf{Z}_{k+1}} \leq  \Trace{\mathbf{Z}_{k+1}^T\mathbf{L}_{\mathbf{Z}_k}\mathbf{Z}_{k+1}}.
	\]
\end{proof}
Lemma~\ref{thm:indlemma} above allows us to show the following corollary:
\begin{cor} 
For fixed point iteration $\mathbf{Z}_{t+1} = \mathbf{H}\mathbf{Z}_t$ optimization of $\mathbf{Z}_{k+1}=\arg\max_{\mathbf{Z}} g(\mathbf{Z}, \mathbf{Z}_k)$, we have $\Trace{\mathbf{Z}_{k+1}\mathbf{L}_{\mathbf{Z}_{k+1}}\mathbf{Z}_{k+1}} \leq  \Trace{\mathbf{Z}_{k}^T\mathbf{L}_{\mathbf{Z}_k}\mathbf{Z}_{k}}$.
\label{cor:indlemmaone}
\end{cor}
\begin{proof}
From Lemma~\ref{thm:indlemma} we have
\begin{align*}
\Trace{\mathbf{Z}_{k+1}\mathbf{L}_{\mathbf{Z}_{k+1}}\mathbf{Z}_{k+1}} 
    \leq  \Trace{\mathbf{Z}_{k+1}^T\mathbf{L}_{\mathbf{Z}_k}\mathbf{Z}_{k+1}}
    \leq \Trace{\mathbf{Z}_{k}^T\mathbf{H}^T\mathbf{L}_{\mathbf{Z}_k}\mathbf{H}\mathbf{Z}_{k}}
\end{align*}
Following approach similar to proof of Lemma~\ref{thm:indlemma} above by substituting eigen decomposition of $\mathbf{H} = \mathbf{Q}\mathbf{\Lambda}\mathbf{Q}^T$ into equation above we get,
\begin{align*}
\Trace{\mathbf{Z}_{k+1}\mathbf{L}_{\mathbf{Z}_{k+1}}\mathbf{Z}_{k+1}} 
\leq \Trace{\mathbf{Z}_{k}^T((\mathbf{Q}^T\mathbf{I}\mathbf{Q})^T)\mathbf{L}_{\mathbf{Z}_k}(\mathbf{Q}^T\mathbf{I}\mathbf{Q})\mathbf{Z}_{k}} 
\leq \Trace{\mathbf{Z}_{k}^T\mathbf{L}_{\mathbf{Z}_k}\mathbf{Z}_{k}}
\end{align*}
\end{proof}
\end{document}